%% file: main.tex
\documentclass[10pt, a4paper, twocolumn, teaser, showabstract]{naverlabseurope}

\usepackage{lipsum}
\usepackage{multicol}
\usepackage{tikz,tkz-kiviat,pgfplots}
\usepackage{cleveref}
\usepackage{graphicx}
\usepackage{booktabs}
\usepackage{rotating}
\usepackage{multirow}

\usepackage{hyperref}

\usepackage{xspace}
\makeatletter
\DeclareRobustCommand\onedot{\futurelet\@let@token\@onedot}
\def\@onedot{\ifx\@let@token.\else.\null\fi\xspace}

\def\eg{\emph{e.g}\onedot} 
\def\ie{\emph{i.e}\onedot}

\def\wrt{w.r.t\onedot} 
\def\etal{\emph{et al}\onedot}
\makeatother

\newcommand{\mypar}[1]{\vspace{.5mm}\noindent \textit{#1}\ }
\newcommand{\duster}[0]{DUSt3R}
\newcommand{\master}[0]{MASt3R}
\newcommand{\masterm}[0]{MASt3R-M}

\newcommand{\orange}[1]{{\color{orange}#1}}

\newcommand{\I}[1]{I^{#1}} %
\newcommand{\cam}[1]{C^#1} %
\newcommand{\C}[1]{C^{#1}} %
\newcommand{\R}{\mathbb{R}} %
\newcommand{\D}[1]{D^{#1}} %
\newcommand{\X}[2]{X^{#1,#2}} %
\newcommand{\Xgt}[2]{\hat{X}^{#1,#2}} %
\newcommand{\z}{z}           %
\newcommand{\zgt}{\hat{z}}   %

\newcommand{\enc}{\text{Encoder}} 
\newcommand{\dec}{\text{Decoder}} 
\newcommand{\head}{\text{Head}_{\text{3D}}}
\newcommand{\headdesc}{\text{Head}_{\text{desc}}}
\newcommand{\lreg}{\ell_{\text{regr}}} %
\newcommand{\lconf}{\mathcal{L}_{\text{conf}}} %
\newcommand{\lmatch}{\mathcal{L}_{\text{match}}} %

\newcommand{\valid}{\mathcal{V}} %
\newcommand{\N}[1]{\text{NN}_{#1}} %
\newcommand{\M}{\mathcal{M}} %
\newcommand{\Mgt}{\hat{\mathcal{M}}} %
\newcommand{\MP}[1]{\mathcal{P}^{#1}} %
\newcommand{\W}[1]{W^{#1}} %

\DeclareMathOperator*{\argmin}{arg\,min}
\newcommand{\un}[0]{\underline }

\graphicspath{{figures/}}

\newtheorem{theorem}{Theorem}[section]
\newtheorem{lemma}[theorem]{Lemma}
\newtheorem{proposition}[theorem]{Proposition}
\newtheorem{corollary}[theorem]{Corollary}

\title{Grounding Image Matching in 3D with \master}

\correspondingauthor{[vincent.leroy,jerome.revaud]@naverlabs.com}

\authors{Vincent Leroy \authsep Yohann Cabon \authsep Jerome Revaud}
\affiliations{NAVER LABS Europe}
\contributions{} 
\website{https://github.com/naver/mast3r}
\websiteref{\href{https://github.com/naver/mast3r}}

\teaserfig{\includegraphics[width=0.95\textwidth]{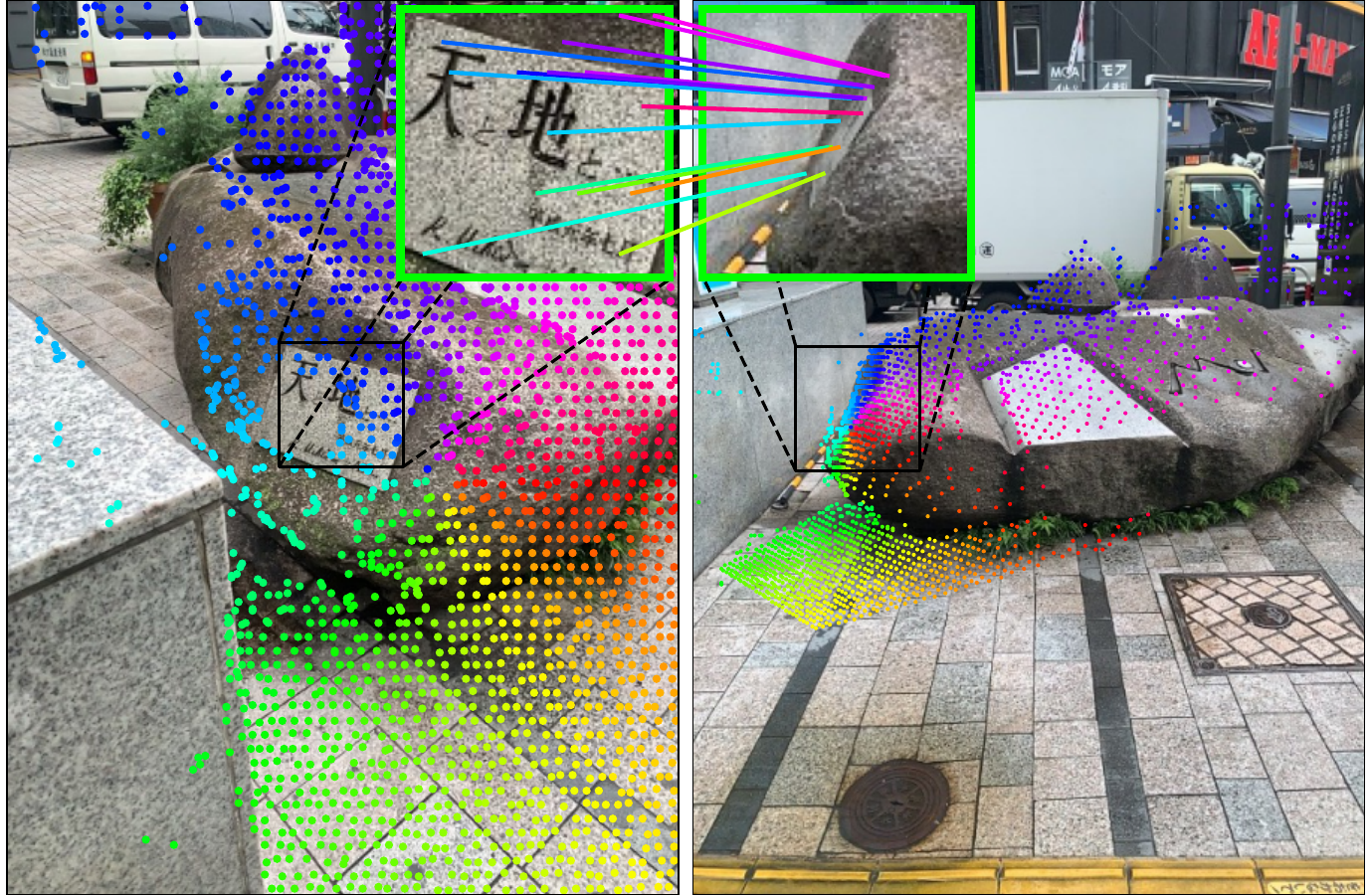}}

\teasercaption{\textbf{Dense Correspondences.} \master{} extends \duster{} as it predicts dense correspondences, even in regions where camera motion significantly degrades the visual similarity. Focal length can be derived from the predicted 3D geometry, making our approach a standalone method for camera calibration, camera pose estimation and 3D scene reconstruction, attaining and improving the performance of the state of the art on several extremely challenging benchmarks. }

\begin{abstract}
Image Matching is a core component of all best-performing algorithms and pipelines in 3D vision. 
Yet despite matching being fundamentally a 3D problem, intrinsically linked to camera pose and scene geometry, it is typically treated as a 2D problem.
This makes sense as the goal of matching is to establish correspondences between 2D pixel fields, but also seems like a potentially hazardous choice.
In this work, we take a different stance and propose to cast matching as a 3D task with \duster{}, a recent and powerful 3D reconstruction framework based on Transformers. 
Based on pointmaps  regression, this method displayed impressive robustness in matching views with extreme viewpoint changes, yet with limited accuracy. 
We aim here to improve the matching capabilities of such an approach while preserving its robustness.  
We thus propose to augment the \duster{} network with a new head that outputs dense local features, trained with an additional matching loss.
We further address the issue of quadratic complexity of dense matching, which becomes prohibitively slow for downstream applications if not carefully treated. We introduce a fast reciprocal matching scheme that not only accelerates matching by orders of magnitude, but also comes with theoretical guarantees and, lastly, yields improved results.
Extensive experiments show that our approach, coined \master{}, significantly outperforms the state of the art on multiple matching tasks.
In particular, it beats the best published methods by 30\% (absolute improvement) in VCRE AUC on the extremely challenging Map-free localization dataset.
\end{abstract}

\begin{document}
\maketitle

\section{Introduction}
\label{sec:intro}

Being able to establish correspondences between pixels across different images of the same scene, %
denoted as \emph{image matching}, 
constitutes a core component of all 3D vision applications, spanning mapping~\cite{mur2015orb, orb-slam3}, localization~\cite{superglue, kapture2020}, navigation~\cite{chaplot2020learning}, photogrammetry~\cite{peppa2018archaeological, culturalheritage} and autonomous robotics in general~\cite{thrun2002probabilistic,ozyecsil2017survey}.
State-of-the-art methods for visual localization, for instance, overwhelmingly rely upon image matching during the offline mapping stage, \eg using COLMAP \cite{colmap}, as well as during the online localization step, typically using PnP~\cite{pnpransac}.
In this paper, we focus on this core task and aim at producing, given two images, a list of pairwise correspondences, denoted as \emph{matches}. 
In particular, we seek to output highly accurate and dense matches that are robust to viewpoint and illumination changes because these are, in the end, the limiting factor for real-world applications~\cite{hartley_zisserman}.

In the past, matching methods have traditionally been cast into a three-steps pipeline consisting of first extracting sparse and repeatable keypoints, then describing them with locally invariant features, and finally pairing the discrete set of keypoints by comparing their distance in the feature space.
This pipeline has several merits: keypoint detectors are precise under low-to-moderate illumination and viewpoint changes, and the sparsity of keypoints makes the problem computationally tractable, enabling very precise matching in milliseconds whenever the images are viewed under similar conditions.
This explains the success and persistence of SIFT~\cite{sift} in 3D reconstruction pipelines like COLMAP~\cite{colmap}.

Unfortunately, keypoint-based methods, by reducing matching to a bag-of-keypoint problem, discard the global geometric context of the correspondence task.
This makes them especially prone to errors in situation with repetitive patterns or low-texture areas, which are in fact ill-posed for local descriptors.
One way to remedy this is to introduce a global optimization strategy during the pairing step, typically leveraging some learned priors about matching, which SuperGlue and similar methods successfully implemented~\cite{superglue,lightglue}.
However, leveraging global context during matching might be too late, if keypoints and their descriptors do not already encode enough information.
For this reason, another direction is to consider dense holistic matching, \ie avoiding keypoints altogether, and matching the entire image at once. 
This recently became possible with the advent of mechanism for global attention~\cite{transformer}. %
Such approaches, like LoFTR~\cite{sun2021loftr}, thus consider images as a whole and the resulting set of correspondences is dense and more robust to repetitive patterns and low-texture areas~\cite{sun2021loftr, jiang2021cotr,revaud16deepmatching, epicflow}.
This led to new state-of-the-art results on the most challenging benchmarks, such as the Map-free localization benchmark~\cite{mapfree}. 

Nevertheless, even a top-performing methods like %
LoFTR~\cite{sun2021loftr} score a relatively disappointing VCRE precision of 34\% on the Map-free localization benchmark.
We argue that this is because, so far, practically all matching approaches have been treating matching as a 2D problem in image space. 
In reality, the formulation of the matching task is intrinsically and fundamentally a 3D problem: pixels that correspond are pixels that observe the same 3D point.  %
Indeed, 2D pixel correspondences and a relative camera pose in 3D space are two sides of the same coin, as they are directly related by the epipolar matrix~\cite{hartley_zisserman}.
Another evidence is that the current top-performer on the Map-free benchmark is \duster{}~\cite{dust3r}, a method initially designed for 3D reconstruction rather than matching, and for which matches are only a by-product of the 3D reconstruction.
Yet, correspondences obtained naively from this 3D output currently outperform all other keypoint- and matching-based methods on the Map-free benchmark.

In this paper, we point out that, while \duster{}~\cite{dust3r} can indeed be used for matching, it is relatively  imprecise, despite being extremely robust to viewpoint changes. %
To remedy this flaw, we propose to attach a second head that regresses dense local feature maps, and train it with an InfoNCE loss.
The resulting architecture, called \master{} for ``Matching And Stereo 3D Reconstruction'' outperforms \duster{} on multiple benchmarks.
To get pixel-accurate matches, we propose a coarse-to-fine matching scheme %
during which matching is performed at several scales.
Each matching step involves extracting reciprocal matches from dense feature maps which, perhaps counter-intuitively, is by far more time consuming than computing the dense feature maps themselves. %
Our proposed solution is a faster algorithm for finding reciprocal matches that is almost two orders of magnitude faster while improving the pose estimation quality.

To summarize, we claim three main contributions.
First, we propose \master{}, a 3D-aware matching approach building on the recently released \duster{} framework. It outputs local feature maps that enable highly accurate and extremely robust matching.
Second, we propose a coarse-to-fine matching scheme associated with a fast matching algorithm, enabling to work with high-resolution images.
Third, \master{} significantly outperform the state-of-the-art on several absolute and relative pose localization benchmarks.

\section{Related works}

\mypar{Keypoint-based matching} has been a cornerstone of computer vision. 
Matching is carried out in three distinct stages: keypoint detection, locally invariant description and nearest-neighbor search in descriptor space.
Departing from the former handcrafted methods like SIFT~\cite{sift,rublee11orb}, modern approaches have been shifting towards learning-based data-driven schemes for detecting keypoints~\cite{laguna19keynet,mishkin18repeatability,verdie15tilde,zhang17discriminative}, describing them~\cite{balntas17hpatches,he18local,tian19sosnet, germain20s2dnet} or both at the same time~\cite{ma21survey,revaud19r2d2,wang21p2net,bhowmik20reinforced,superpoint,aslfeatCVPR20}. 
Overall, keypoint-based approaches are predominant in many benchmarks~\cite{balntas17hpatches, schoenberger17,ltlv,jin19}, underscoring their enduring value in tasks requiring high precision and speed~\cite{csurka18,schoenberger17}. %
One notable issue, however, is they reduce matching to a local problem, \ie discarding its holistic nature.
SuperGlue and similar approaches~\cite{superglue, lightglue} thus propose to perform global reasoning in the last pairing step leveraging stronger priors to guide matching, yet leaving the detection and description local. %
While successful, it is still limited by the local nature of keypoints and their inability to remain invariant to strong viewpoint changes.

\mypar{Dense matching.}
In contrast to keypoint-based approaches, semi-dense~\cite{sun2021loftr, jiang2021cotr, bokman22, aspanformer22,tang2022quadtree, pats} and dense approaches~\cite{edstedt2023dkm,roma,dgcnet18, glunet, truong21,pmatch,pdcnetp,dfm21} offer a different paradigm for establishing image correspondences, considering all 
possible pixel associations. 
Very reminiscent of optical flow approaches~\cite{raft, flowformer, flowformerpp, videoflow, dong2023rethinking, gma21}, they are usually employing coarse-to-fine schemes to decrease computational complexity. %
Overall, these methods aim to consider matching from a global perspective, %
at the cost of increased computational resources. 
Dense matching has proven effective in scenarios where detailed spatial relationships and textures are critical for understanding scene geometry, leading to top performance on many benchmarks~\cite{imc22,imc23,mapfree,wxbs,sun2021loftr,superglue} %
that are especially challenging for keypoints due to extreme changes in viewpoint or illumination. 
These approaches still cast matching as a 2D problem, which limits their usage for visual localization.

\mypar{Camera Pose estimation}
techniques vary widely, but the most successful strategies, for speed, accuracy and robustness trade-off, are fundamentally based on pixel matching~\cite{colmap, hloc, vsfm}. 
The constant improvement of matching methods has fostered the introduction of more challenging camera pose estimation benchmarks, such as Aachen Day-Night, InLoc, CO3D or Map-free~\cite{inloc,aachen,co3d,mapfree}, all featuring strong viewpoint and/or illumination changes.
The most challenging of them is undoubtedly Map-free~\cite{mapfree}, a localization dataset for which a single reference image is provided but no map, with viewpoint changes up to 180$^\circ$.

\mypar{Grounding matching in 3D}
thus becomes a crucial necessity in these challenging conditions where classical 2D-based matching utterly falls short.
Leveraging priors about the physical properties of the scene in order to improve accuracy or robustness has been widely explored in the past,
but most previous works settle for leveraging epipolar constraints for semi-supervised learning of correspondences without any fundamental change~\cite{yang19dataadaptive,patch2pix,yao19monet,caps_epipolar_2020,epipolar_vit_2020,bhalgat_epipolar_2023,kloepfer_scenes_2024,inductive_bias_epipolar_2022}. 
Toft \etal~\cite{mono_depth_helps20}, on its part, propose to improve keypoint descriptors by rectifying images with perspective transformations obtained from an off-the-shelf monocular depth predictor. 
Recently, diffusion for pose~\cite{posediffusion} or rays~\cite{raydiffusion}, although not matching approaches strictly speaking, show promising performance by incorporating 3D geometric constraints into their pose estimation formulation. %
Finally, the recent~\duster{}~\cite{dust3r} explore the possibility of recovering correspondences from the \emph{a-priori} harder task of 3D reconstruction from uncalibrated images.
Despite not being trained explicitly for matching, this approach yields promising results, topping the Map-free leaderboard~\cite{mapfree}.
Our contribution is to pursue this idea, by regressing local features and explicitly training them for pairwise matching.

\section{Method}

\begin{figure*}[t]
    \centering
    \includegraphics[width=\linewidth, trim=0 348 0 0, clip]{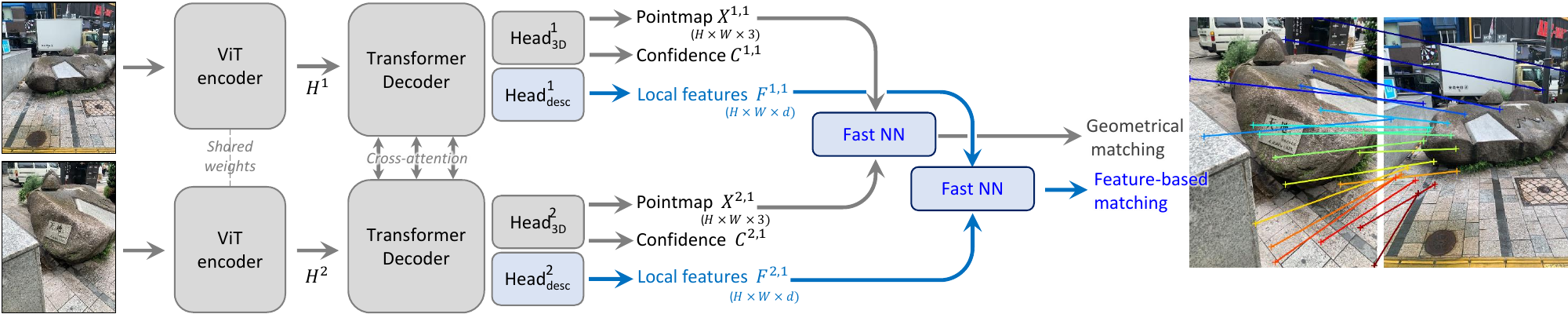}
    \caption{
        \text{Overview of the proposed approach.}
        Given two input images to match, our network regresses for each image and each input pixel a 3D point, a confidence value and a local feature.
        Plugging either 3D points or local features into our fast reciprocal NN matcher (\ref{sub:fast}) yields robust correspondences.
        Compared to the \duster{} framework which we build upon, our contributions are highlighted in \color{blue}{blue}.
    }
    \label{fig:overview}
\end{figure*}

Given two images $\I{1}$ and $\I{2}$, respectively captured by two cameras $\cam{1}$ and $\cam{2}$ with unknown %
parameters, we wish to recover a set of pixel correspondences $\{(i,j)\}$ where $i,j$ are pixels $i=(u_i,v_i), j=(u_j,v_j) \in \{1,\ldots,W\}\times \{1,\ldots,H\}$, $W,H$ being the respective width and height of the images. We assume they have the same resolution for the sake of simplicity, yet without loss of generality. The final network can handle pairs of variable aspect ratios.

Our approach, illustrated in \cref{fig:overview}, aims at jointly performing 3D scene reconstruction and matching given two input images.
It is based on the \duster{} framework recently proposed by Wang \etal~\cite{dust3r}, which we first review in \cref{sub:duster} before presenting our proposed matching head and its corresponding loss in \cref{sub:head}.
We then introduce an optimized matching scheme specially devised to deal with dense feature maps in \ref{sub:fast},  that we use for coarse-to-fine matching in \cref{sub:c2f}.

\subsection{The \duster{} framework}
\label{sub:duster}

\duster~\cite{dust3r} is a recently proposed approach that jointly solves the calibration and 3D reconstruction problems from images alone. 
A transformer-based network predicts a \emph{local} 3D reconstruction given two input images, in the form of two dense 3D point-clouds $\X{1}{1}$ and $\X{2}{1}$, denoted as \emph{pointmaps} in the following.

A pointmap $\X{a}{b} \in \R^{H \times W \times 3}$ represents a dense 2D-to-3D mapping between each pixel $i=(u,v)$ of the image $\I{a}$ and its corresponding 3D point $\X{a}{b}_{u,v} \in \R^3$ expressed in the coordinate system of camera $\cam{b}$.
By regressing two pointmaps $\X{1}{1}, \X{2}{1}$ expressed in the \emph{same} coordinate system of camera $\cam{1}$, \duster{} effectively solves the joint calibration and 3D reconstruction problem.
In the case where more than two images are provided, a second step of global alignment merges all pointmaps in the same coordinate system. Note that, in this paper, we do not make use of this step and restrict ourselves to the binocular case. %
We now explain the inference in more details.

Both images are first encoded in a Siamese manner with a ViT~\cite{vit}, yielding two representations $H^1$ and $H^2$:
\begin{align}
    H^1 & = \enc(\I1), \\
    H^2 & = \enc(\I2). 
\end{align} 
Then, two intertwined decoders process these representations jointly, exchanging information via cross-attention to `understand' the spatial relationship between viewpoints and the global 3D geometry of the scene. The new representations augmented with this spatial information are denoted as $H^1$ and $H^2$:
\begin{equation}
    H'^1, H'^2 = \dec(H^1, H^2).
\end{equation} 
Finally, two prediction heads regress the final pointmaps and confidence maps from the concatenated representations output by the encoder and decoder:
\begin{align}
    \X{1}{1}, \C1 & = \head^1([H^1,H'^1]), \\
    \X{2}{1}, \C2 & = \head^2([H^2,H'^2]).
\end{align} 

\mypar{Regression loss.}
\duster{} is trained in a fully-supervised manner using a simple regression loss
\begin{equation}
    \lreg(v,i) = \left\Vert \frac{1}{\z}\X{v}{1}_{i}  - \frac{1}{\zgt}\Xgt{v}{1}_{i} \right\Vert, \label{eq:reg}
\end{equation} 
where $v \in \{1,2\}$ is the view and $i$ is a pixel for which the ground-truth 3D point $\Xgt{v}{1} \in \R^3$ is defined.
In the original formulation, normalizing factors $\z, \zgt$ are introduced to make the reconstruction invariant to scale.
These are simply defined as the mean distance of all valid 3D points to the origin. %

\mypar{Metric predictions.}
In this work, we note that scale invariance is not necessarily desirable, as some potential use-cases like map-free visual localization necessitates metric-scale predictions.
Therefore, we modify the regression loss to ignore normalization for the predicted pointmaps when the ground-truth pointmaps are known to be metric.
That is, we set $\z := \zgt$ whenever ground-truth is metric, so that $\lreg(v,i) = || \X{v}{1}_{i}  - \Xgt{v}{1}_{i} || / \zgt$ in this case.
As in \duster~\cite{dust3r}, the final confidence-aware regression loss is defined as 
\begin{equation}
    \lconf = \sum_{v \in \{1,2\}} \, \sum_{i \in \valid^v} 
            \C{v}_i \lreg(v,i) - \alpha \log \C{v}_i.
\end{equation} 

\subsection{Matching prediction head and loss}
\label{sub:head}

To obtain reliable pixel correspondences from pointmaps, a standard solution is to look for reciprocal matches in some invariant feature space \cite{dust3r,d2net,sethi87,wu07insitu}.
While such a scheme works remarkably well with \duster{}'s regressed pointmaps (\ie in a 3-dimensional space) even in presence of extreme viewpoint changes, we note that the resulting correspondences are rather imprecise, yielding suboptimal accuracy.
This is a rather natural result as (i) regression is inherently affected by noise, and (ii) because \duster{} was never explicitly trained for matching.

\mypar{Matching head.}
For these reasons, we propose to add a second head that outputs two dense feature maps $\D{1}$ and $\D{2} \in \R^{H \times W \times d}$ of dimensional $d$:
\begin{align}
    \D{1} & = \headdesc^1([H^1,H'^1]), \\
    \D{2} & = \headdesc^2([H^2,H'^2]).
\end{align} 
We implement the head as a simple 2-layers MLP interleaved with a non-linear GELU activation function~\cite{gelu}. 
Lastly, we normalize each local feature to unit norm.
More details can be found in the supplementary material.

\mypar{Matching objective.}
We wish to encourage each local descriptor from one image to match with at most a single descriptor from the other image that represents the same 3D point in the scene.
To that aim, we leverage the infoNCE~\cite{infonce} loss over the set of ground-truth correspondences $\Mgt=\{(i,j)|\Xgt{1}{1}_i = \Xgt{2}{1}_j\}$:
\begin{align}
    \lmatch & = - \sum_{(i,j)\in \Mgt} 
        \log\frac{s_\tau(i,j)}{\sum_{k\in \MP1} s_\tau(k,j)} +
        \log\frac{s_\tau(i,j)}{\sum_{k\in \MP2} s_\tau(i,k)}, \label{eq:matching_loss} \\
    & \text{with } s_\tau(i,j) = \exp\left[ -\tau \D{1\top}_i \D2_j \right]. %
\end{align} 
Here, $\MP1 = \{i | (i,j)\in\Mgt\}$ and $\MP2 = \{j | (i,j)\in\Mgt\}$ denote the subset of considered pixels in each image and $\tau$ is a temperature hyper-parameter.
Note that this matching objective is essentially a cross-entropy \emph{classification} loss: contrary to regression in~\cref{eq:reg}, the network is only rewarded if it gets the correct pixel right, not a nearby pixel. 
This strongly encourages the network to achieve high-precision matching.
Finally, both regression and matching losses are combined to get the final training objective:
\begin{equation}
    \mathcal{L}_{\text{total}} = \lconf + \beta \lmatch
\end{equation} 

\subsection{Fast reciprocal matching}
\label{sub:fast}

Given two predicted feature maps $\D{1}, \D{2} \in \R^{H \times W \times d}$, we aim to extract a set of reliable pixel correspondences, \ie mutual nearest neighbors of each others:
\begin{align}
    \M = \{(i,j) \ | \  j=\N{2}(\D{1}_i) \text{ and } i=\N{1}(\D{2}_j) \}, \label{eq:corres}\\
    \text{with } \N{A}(\D{B}_j) = \argmin_{i} \left\Vert \D{A}_i - \D{B}_j \right\Vert. \label{eq:nn}
\end{align} 

Unfortunately, naive implementation of reciprocal matching has a high computational complexity of $O(W^2 H^2)$, since every pixel from an image must be compared to every pixels in the other image.
While optimizing the nearest-neighbor (NN) search is possible, \eg using K-d trees~\cite{scipy}, this kind of optimization becomes typically very inefficient in high dimensional feature space and, in all cases, orders of magnitude slower than the inference time of \master{} to output $\D1$ and $\D2$. %

\mypar{Fast matching.}
We therefore propose a faster approach based on sub-sampling.
It is based on an iterated process that starts from an initial sparse set of $k$ pixels $U^0=\{U^0_n\}_{n=1}^{k}$, typically sampled regularly on a grid in the first image $\I1$.
Each pixel is then mapped to its NN on $\I2$, yielding $V^1$, and the resulting pixels are mapped back again to $\I1$ in the same way:
\begin{equation}
    U^t \longmapsto [ \N{2}(\D{1}_u) ]_{u \in U^t} \equiv V^t \longmapsto [ \N{1}(\D{2}_v) ]_{v \in V^t} \equiv U^{t+1}
\end{equation}
The set of reciprocal matches (those which form a cycle, \ie $\M_k^t = \{ (U^t_n, V^t_n) \ |$ $\ U^t_n=U^{t+1}_n \}$) are then collected.
For the next iteration, pixels that already converged are filtered out, \ie updating $U^{t+1} := U^{t+1} \setminus U^t$.
Likewise, starting from $t=1$ we also verify and filter $V^{t+1}$, comparing it with $V^t$ in a similar fashion.
As illustrated in \cref{fig:recip} (left), this process is then iterated a fixed number of times, until most correspondences converge to stable (reciprocal) pairs. 
In~\cref{fig:recip} (center), we show that the number of un-converged point $|U^t|$ rapidly decreases to zero after a few iterations.
Finally, the output set of correspondences consists of the concatenation of all reciprocal pairs $\M_k = \bigcup_t \M_k^t$.

\mypar{Theoretical guarantees.}
The overall complexity of the fast matching is $O(kWH)$, which is $WH/k \gg 1$ times faster than the naive approach denoted \emph{all}, as illustrated in \cref{fig:recip} (right).
It is worth pointing out that our fast matching algorithm extracts a \emph{subset} of the full set $\M$, which is bounded in size by $|\M_k| \leq k$.
We study in the supplementary material the convergence guarantees of this algorithm and how it evinces outlier-filtering properties, which explains why the end accuracy is actually \emph{higher} than when using the full correspondence set $\M$, see \cref{fig:recip} (right).

\begin{figure*}[t]
    \centering
    \includegraphics[width=0.38\linewidth, trim=0 347 630 0, clip]{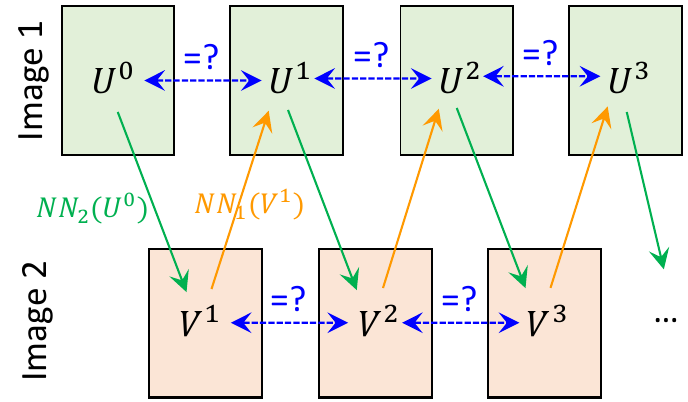}
    \includegraphics[width=0.3\linewidth]{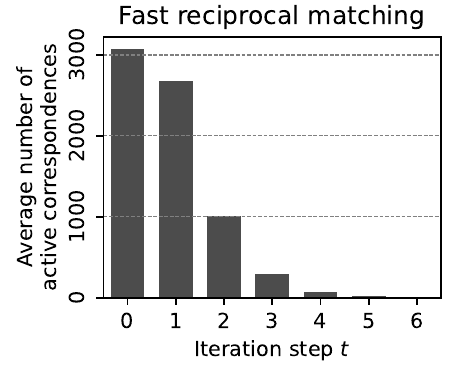}
    \includegraphics[width=0.3\linewidth]{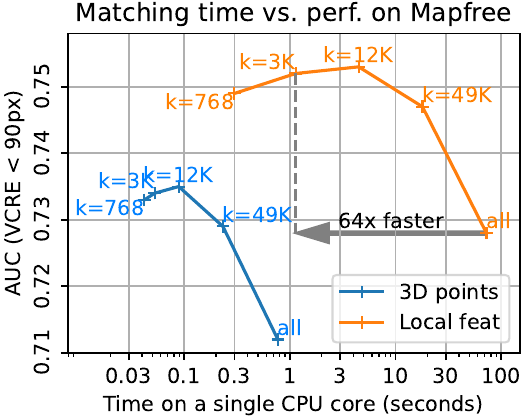}
    \caption{Fast reciprocal matching. 
        \textbf{Left}: Illustration of the fast matching process, starting from an initial subset of pixels $U^0$ and propagating it iteratively using $NN$ search. Searching for cycles (blue arrows) detect reciprocal correspondences and allows to accelerate the subsequent steps, by removing points that converged.
        \textbf{Center}: Average number of remaining points in $U^t$ at iteration $t=1\ldots6$.
        After only 5 iterations, nearly all points have already converged to a reciprocal match.
        \textbf{Right}: Performance-versus-time trade-off on the Map-free dataset. 
        Performance actually improves, along with matching speed, when performing moderate levels of subsampling.
    }
    \label{fig:recip}
\end{figure*}

\subsection{Coarse-to-fine matching}
\label{sub:c2f}

Due to the quadratic complexity of attention \wrt the input image area ($W\times H$), \master{} only handles images of 512 pixels in their largest dimension. Larger images would require significantly more compute power to train, and ViTs do not generalize yet to larger test-time resolutions~\cite{dpt,dino_v2}.   
As a result, high-resolution images (\eg 1M pixel) needs to be downscaled to be matched, afterwards the resulting correspondences are upscaled back to the original image resolution.
This can lead to some performance loss, sometimes sufficient to cause substantial degradation in term of localization accuracy or reconstruction quality.

\mypar{Coarse-to-fine matching}
is a standard technique to preserve the benefit of matching high-resolution images with a lower-resolution algorithm~\cite{raft, ranjan17}. We thus explore this idea for \master{}.
Our procedure starts with performing matching on downscaled versions of the two images. 
We denote the set of coarse correspondences obtained with subsampling $k$ as $\M^0_k$.
Next, we generate a grid of overlapping window crops $\W{1}$ and $\W{2} \in \R^{w \times 4}$ on each full-resolution image independently. 
Each window crop measures 512 pixels in its largest dimension and contiguous windows overlap by 50\%. 
We can then enumerate the set of all window pairs $(w_1,w_2) \in \W{1}\times \W{2}$, from which we select a subset covering most of the coarse correspondences $\M^0_k$.
Specifically, we add window pairs one by one in a greedy fashion until 90\% of correspondences are covered.
Finally, we perform matching for each window pair independently:
\begin{align}
    \D{w_1}, \D{w_2} & = \text{\master{}}(\I{1}_{w_1}, \I{2}_{w_2}) \\
    \M_k^{w_1,w_2} & = \text{fast\_reciprocal\_NN}(\D{w_1}, \D{w_2})
\end{align}\\ 
Correspondences obtained from each window pair are finally mapped back to the original image coordinates and concatenated, thus providing dense full-resolution matches.

\section{Experimental results}
\label{sec:xps}

We detail in~\cref{ssec:training} the training procedure of \master{}.
Then, we evaluate on several tasks, each time comparing with the state of the art, starting with visual camera pose estimation on the Map-Free Relocalization Benchmark~\cite{mapfree} (\cref{ssec:mapfree}), the CO3D and RealEstate datasets (\cref{ssec:relpose}) and other standard Visual Localization benchmarks in~\cref{ssec:visloc}. 
Finally, we leverage~\master{} for Dense Multi-View Stereo (MVS) reconstruction in~\cref{ssec:mvs}.
\subsection{Training}
\label{ssec:training}
\mypar{Training data}.
We train our network with a mixture of 14 datasets:
Habitat~\cite{Savva_2019_ICCV}, 
ARKitScenes~\cite{arkitscenes}, 
Blended MVS~\cite{blendedMVS}, 
MegaDepth~\cite{megadepth}, 
Static Scenes 3D~\cite{MIFDB16}, 
ScanNet++~\cite{scannet++}, 
CO3D-v2~\cite{co3d}, 
Waymo~\cite{Sun_2020_CVPR},
Map-free~\cite{mapfree}, 
WildRgb~\cite{wildrgb},
VirtualKitti~\cite{vkitti2},
Unreal4K~\cite{unreal4k},
TartanAir~\cite{tartanair2020iros} and
an internal dataset.
These datasets feature diverse scene types: indoor, outdoor, synthetic, real-world, object-centric, etc. 
Among them, 10 datasets have metric ground-truth.
When image pairs are not directly provided with the dataset, we extract them based on the method described in~\cite{croco_v2}. 
Specifically, we utilize off-the-shelf image retrieval and point matching algorithms to match and verify image pairs. 

\mypar{Training.}
We base our model architecture on the public \duster{} model~\cite{dust3r} and use the same backbone (ViT-Large encoder and ViT-Base decoder).
To benefit the most from \duster{}'s 3D matching abilities, we initialize the model weights to the publicly available \duster{} checkpoint. %
During each epoch, we randomly sample 650k pairs equally distributed between all datasets.
We train our network for 35 epoch with a cosine schedule and initial learning rate set to $0.0001$.
Similar to~\cite{dust3r}, we randomize the image aspect ratio at training time, ensuring that the largest image dimension is 512 pixels.
We set the local feature dimension to $d=24$ and the matching loss weight to $\beta=1$.
It is important that the network sees different scales at training time, because coarse-to-fine matching starts from zoomed-out images to then zoom-in on details (see \cref{sub:c2f}).
We therefore perform aggressive data augmentation during training in the form of random cropping.
Image crops are transformed with a homography to preserve the central position of the principal point. 

\mypar{Correspondence sampling.}
To generate ground-truth correspondences necessary for the matching loss (\cref{eq:matching_loss}), 
we simply find reciprocal correspondences between on the ground-truth 3D pointmaps $\Xgt{1}{1} \leftrightarrow \Xgt{2}{1}$.
We then randomly subsample 4096 correspondences per image pairs. 
If we cannot find enough correspondences, we pad with random false correspondences so that the likelihood of finding a true match remains constant.

\mypar{Fast nearest neighbors.}
For the fast reciprocal matching from \cref{sub:fast}, we implement the nearest neighbor function $\N{}(x)$ from \cref{eq:nn} differently depending on the dimension of $x$.
When matching 3D points $x \in \R^3$, we implement $\N{}(x)$ using K-d trees~\cite{kdtree}.
For matching local features with $d=24$, however, K-d trees become highly inefficient due to the curse of dimensionality~\cite{curse_dim}.
Therefore, we rely on the optimized FAISS library~\cite{faiss, faiss-gpu} in this case. %

\subsection{Map-free localization}
\label{ssec:mapfree}
\mypar{Dataset description.}
We start our experiments with the Map-free relocalization benchmark~\cite{mapfree}, an extremely challenging dataset aiming at localizing the camera in metric space given a single reference image without any map.
It comprises a training, validation and test sets of 460, 65 and 130 scenes resp., each featuring two video sequences. %
Following the benchmark, we evaluate in term of Virtual Correspondence Reprojection Error (VCRE) and camera pose accuracy, see~\cite{mapfree} for details.

\mypar{Impact of subsampling.}
We do not resort to coarse-to-fine matching for this dataset, as the image resolution is already close to \master{} working resolution ($720\times540$ vs. $512\times384$ resp.).
As mentioned in \cref{sub:fast}, computing dense reciprocal matching is prohibitively slow even with optimized code for searching nearest neighbors.
We therefore resort to subsampling the set of reciprocal correspondences, keeping at most $k$ correspondences from the complete set $\M$ (\cref{eq:corres}).
\cref{fig:recip} (right) shows the impact of subsampling in term of AUC (VCRE) performance and timing.
Surprisingly, the performance significantly \emph{improves} for intermediate values of subsampling.
Using $k=3000$, we can accelerate matching by a factor of 64 while significantly improving the performance.
We provide insights in the supplementary material regarding this phenomenon.
Unless stated otherwise, we keep $k=3000$ for subsequent experiments.

\mypar{Ablations on losses and matching modes.}
We report results on the validation set in \cref{tab:mapfree} for different variants of our approach:
\duster{} matching 3D points (I); \master{} also matching 3D points (II) or local features (III, IV, V). 
For all methods, we compute the relative pose from the essential matrix~\cite{hartley_zisserman} estimated with the set of predicted matches (PnP performs similarly).
The metric scene scale is inferred from the depth extracted with an off-the-shelf DPT finetuned on KITTI~\cite{dpt} (I-IV) or from the depth directly output by \master{} (V).

First, we note that all proposed methods significantly outperforms the \duster{} baseline, probably because \master{} is trained longer and with more data. 
All other things being equal, matching descriptors perform significantly better than matching 3D points (II versus IV). 
This confirms our initial analysis that regression is inherently unsuited to compute pixel correspondences, see \cref{sub:head}.

We also study the impact of training only with a single matching objective ($\lmatch$ from \cref{eq:matching_loss}, III).
In this case, the performance overall degrades compared to training with both 3D and matching losses (IV), in particular in term of pose estimation accuracy (\eg median rotation of $10.8^\circ$ for (III) compared to $3.0^\circ$ for (IV)).
We point out that this is in spite of the decoder now having \emph{more capacity to carry out a single task}, instead of two when performing 3D reconstruction simultaneously, indicating that grounding matching in 3D is indeed crucial to improve matching.
Lastly, we observe that, when using metric depth directly output by \master{}, the performance largely improves. 
This suggests that, as for matching, the depth prediction task is largely correlated with 3D scene understanding, and that the two tasks strongly benefit from each other.

\begin{table*}[t]
    \caption{Results on the \emph{validation} set of the Map-free dataset. (\textbf{First} and \un{second} best)}
    \label{tab:mapfree}
    \centering
    \resizebox{0.9\linewidth}{!}{
    \begin{tabular}{@{}cllcrlclclrllclc@{}}
    \toprule
      & & \multirow{2}{*}{\begin{turn}{90} \scriptsize match \end{turn}} &   &  \multicolumn{5}{c}{VCRE (<90px)} &  & \multicolumn{6}{c}{Pose Error} \\
      \cline{5-9} \cline{11-16}
      & & & depth &  Reproj. $\downarrow$ &  \quad\quad&Prec. $\uparrow$ &  \quad\quad&AUC $\uparrow$ &  \quad\quad\quad&\multicolumn{2}{c}{Med. Err. (m,°)}$\downarrow$ &  \quad\quad&Precision $\uparrow$ &  \quad\quad&AUC $\uparrow$ \\
    \midrule
     (I)&\duster{}  &3d& DPT &  125.8 px&  &45.2\% &  &0.704 &   &1.10m& 9.4°&  &17.0\% &  &0.344 \\
     (II)&\master{} &3d& DPT &  112.0 px&  &49.9\% &  &0.732 &   &0.94m& 3.6°&  &21.5\% &  &0.409 \\
     (III)&\masterm{}  &feat& DPT &  \un{107.7} px&  &\un{51.7\%} &  &0.744 &   &1.10m& 10.8°&  &19.3\% &  &0.382 \\
     (IV)&\master{}  &feat& DPT &  112.9 px&  &51.5\% &  &\un{0.752} &   &\un{0.93m}& \un{3.0}°&  &\un{23.2\%} &  &\un{0.435} \\
    \midrule
     (V)&\master{}  &feat& (auto) &  \textbf{57.2} px&  &\textbf{75.9\%} &  &\textbf{0.934} &   &\textbf{0.46m}& \textbf{3.0}°&  &\textbf{51.7\%} &  &\textbf{0.746} \\
    \bottomrule
    \end{tabular}
    }
\end{table*}

\mypar{Comparisons}
on the test set is reported in \cref{tab:mapfree_sota}.
Overall, \master{} outperforms all state-of-the-art approaches by a large margin, achieving more than 
93\% in VCRE AUC.
This is a 30\% absolute improvement compared to the second best published method, LoFTR+KBR~\cite{sun2021loftr,spencer2024cribstv}, that get 63.4\% in AUC. 
Likewise, the median translation error is vastly reduced to 36cm, compared to approx. 2m for the state-of-the-art methods. 
A large part of the improvement is of course due to \master{} predicting metric depth, but note that our variant leveraging depth from DPT-KITTI (thus purely matching-based) outperforms all state-of-the-art approaches as well.

We also provide the results of direct regression with \master{}, \ie without matching, simply using PnP on the pointmap $\X{2}{1}$ of the second image.
These results are surprisingly on par with our matching-based variant, even though the ground-truth calibration of the reference camera is not used. 
As we show below, this does not hold true for other localization datasets, and computing the pose via matching (\eg with PnP or essential matrix) with known intrinsics seems safer in general.

\begin{table*}[t]
    \caption{Comparison with the state of the art on the \emph{test} set of the Map-free dataset.}
    \label{tab:mapfree_sota}
    \centering
    \resizebox{0.9\linewidth}{!}{
    \begin{tabular}{@{}lcrlclclrllclc@{}}
    \toprule
     &  &  \multicolumn{5}{c}{VCRE (<90px)} &  & \multicolumn{6}{c}{Pose Error} \\
     \cline{3-7}  \cline{9-14}
     & depth & Reproj. $\downarrow$ &  \quad\quad&Prec. $\uparrow$ &  \quad\quad&AUC $\uparrow$ &  \quad\quad\quad&\multicolumn{2}{c}{Med. Err. (m,°)}$\downarrow$ &  \quad\quad&Precision $\uparrow$ &  \quad\quad&AUC $\uparrow$ \\
    \midrule
    RPR~\cite{mapfree} & DPT &  147.1 px&  &40.2\% &  &0.402 &   &1.68m& 22.5°&  &6.0\% &  &0.060 \\
    SIFT~\cite{sift} & DPT &  222.8 px&  &25.0\% &  &0.504 &   &2.93m& 61.4°&  &10.3\% &  &0.252 \\
    SP+SG~\cite{superglue} & DPT &  160.3 px&  &36.1\% &  &0.602 &   &1.88m& 25.4°&  &16.8\% &  &0.346 \\
    LoFTR~\cite{sun2021loftr} & KBR &  165.0 px&  &34.3\% &  &0.634 &   &2.23m& 37.8°&  &11.0\% &  &0.295 \\
    \duster{}~\cite{dust3r} & DPT &  116.0 px&  &50.3\% &  &0.697 &   &0.97m& 7.1°&  &21.6\% &  &0.394 \\
    \midrule
    \master{} & DPT &  104.0 px&  &54.2\% &  &0.726 &   &0.80m& \bf 2.2°&  &27.0\% &  &0.456 \\
    \master{} & (auto) &  \textbf{48.7} px&  &\textbf{79.3\%} &  &0.933 &   &\textbf{0.36}m& \textbf{2.2}°&  &\textbf{54.7\%} &  &0.740 \\
    \multicolumn{2}{l}{\master{} (direct reg.)} &  53.2 px&  &79.1\% &  &\textbf{0.941} &   &0.42m& 3.1°&  &53.0\% &  &\textbf{0.777} \\
    \bottomrule
    \end{tabular}
    }
\end{table*}

\mypar{Qualitative results.}
We show in \cref{fig:quali} some matching results for pairs with strong viewpoint change (up to 180$^\circ$). 
We also highlight with insets some specific regions that are correctly matched by \master{} in spite of drastic appearance changes. 
We believe these correspondences to be nearly impossible to get with 2D-based matching methods. 
In contrast, grounding the matching in 3D allows to solve the issue relatively straightforwardly.

\begin{figure*}[t]
    \centering
    \resizebox{\linewidth}{!}{
    \includegraphics[width=\linewidth]{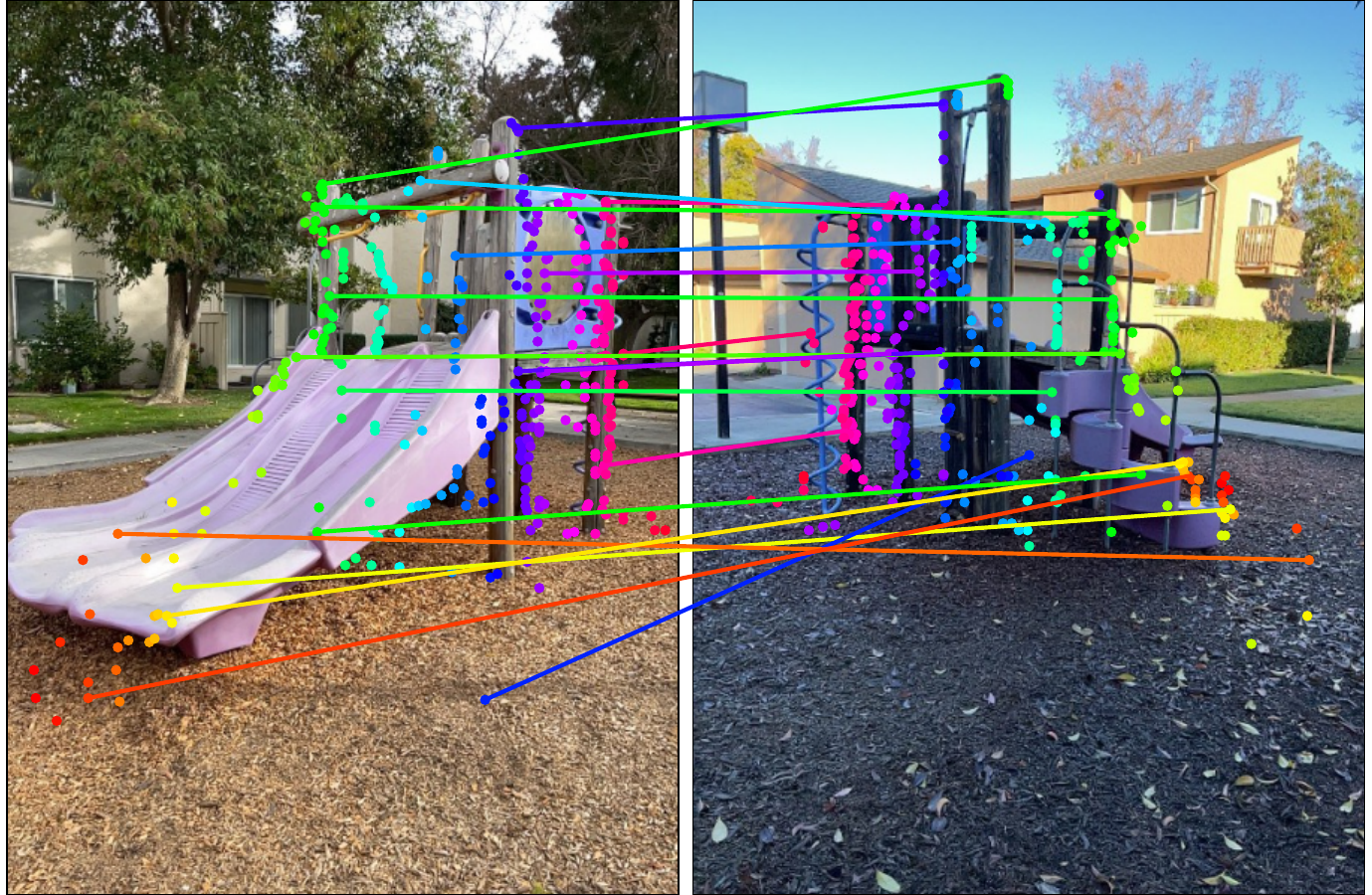}
    \quad
    \includegraphics[width=\linewidth]{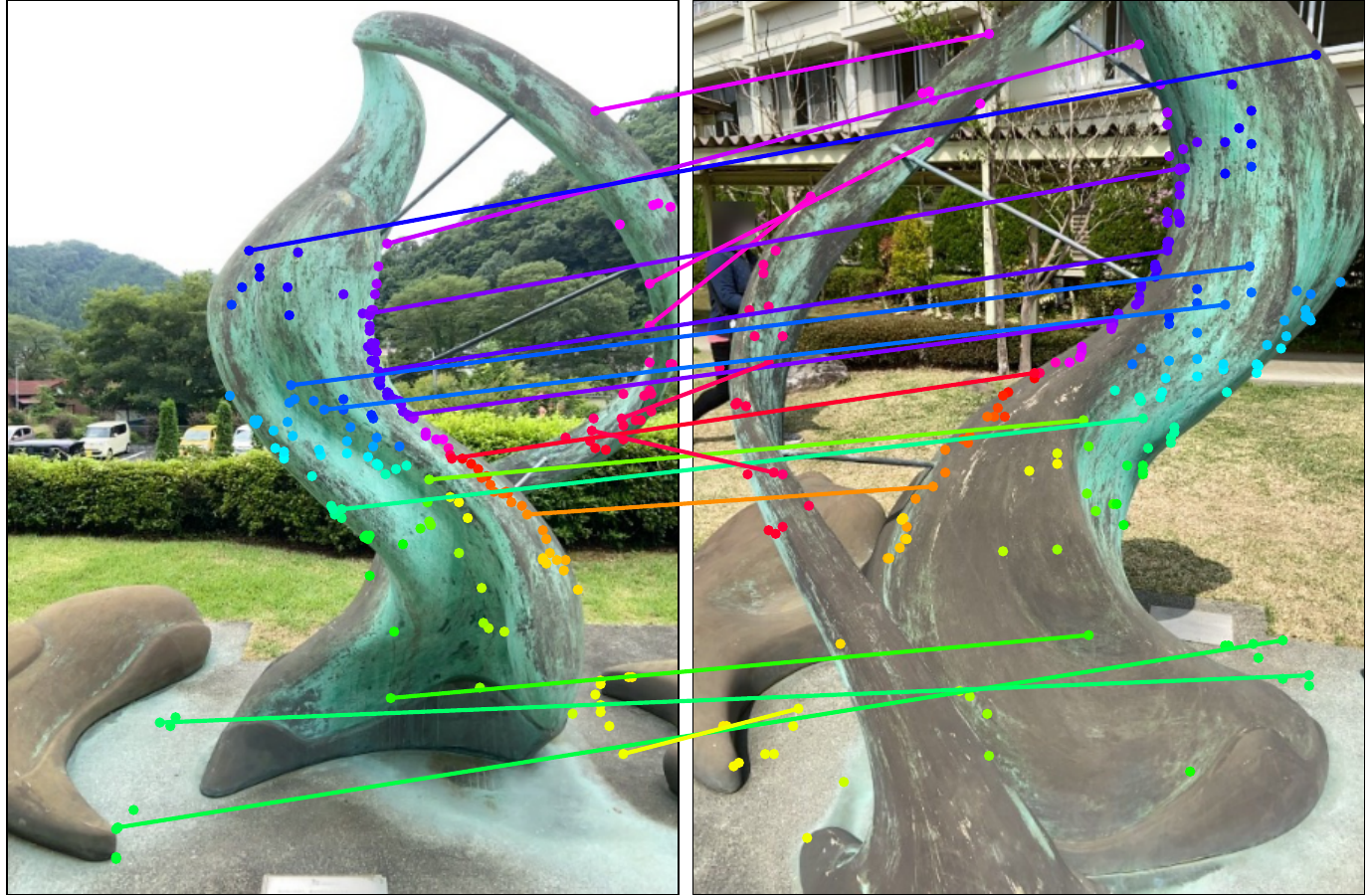}
    \quad
    \includegraphics[width=\linewidth]{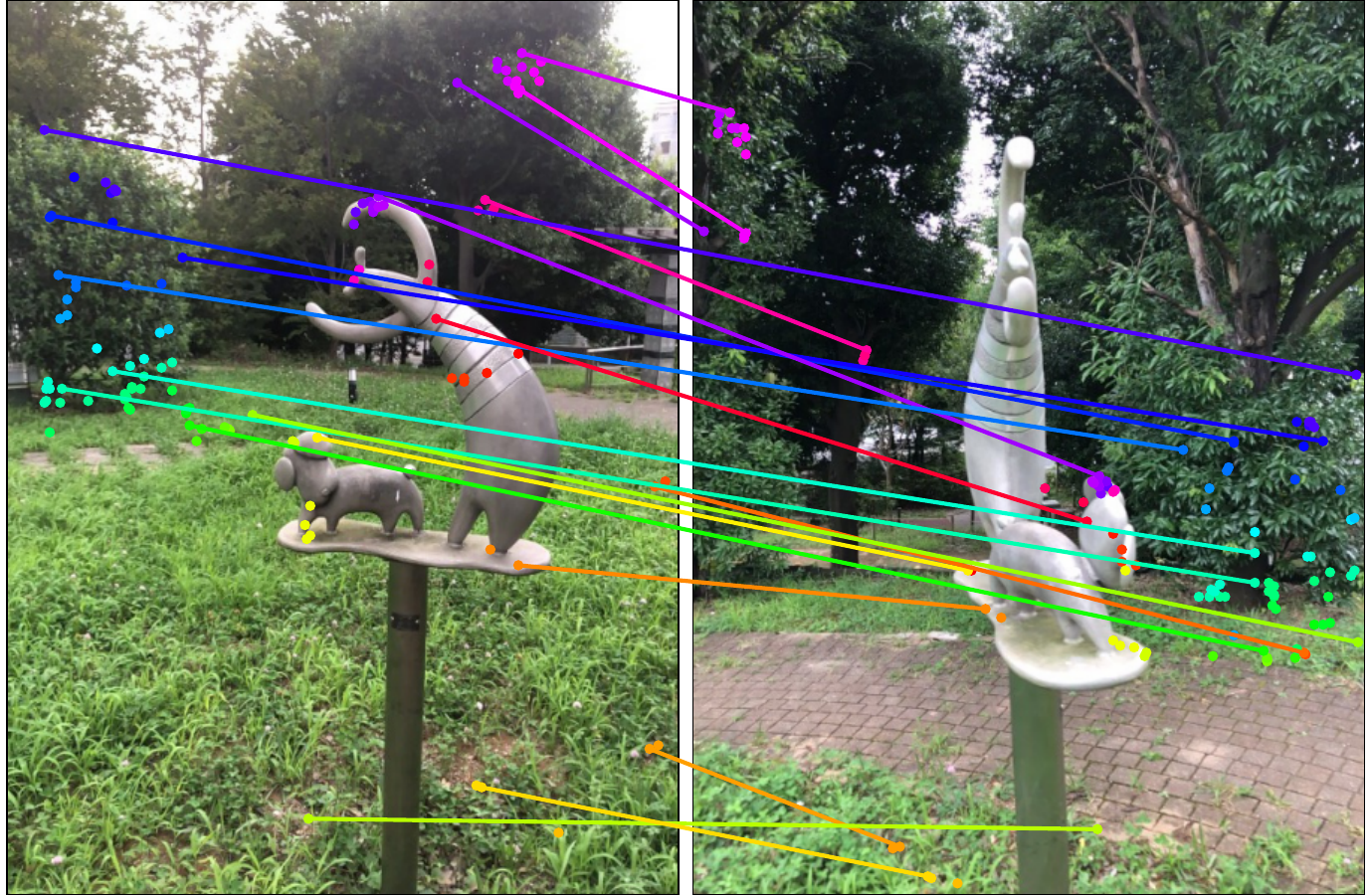}
    }
    \resizebox{\linewidth}{!}{
    \includegraphics[width=\linewidth]{figures/quali/mapfree/s00646_0_400.pdf}
    \quad
    \includegraphics[width=\linewidth]{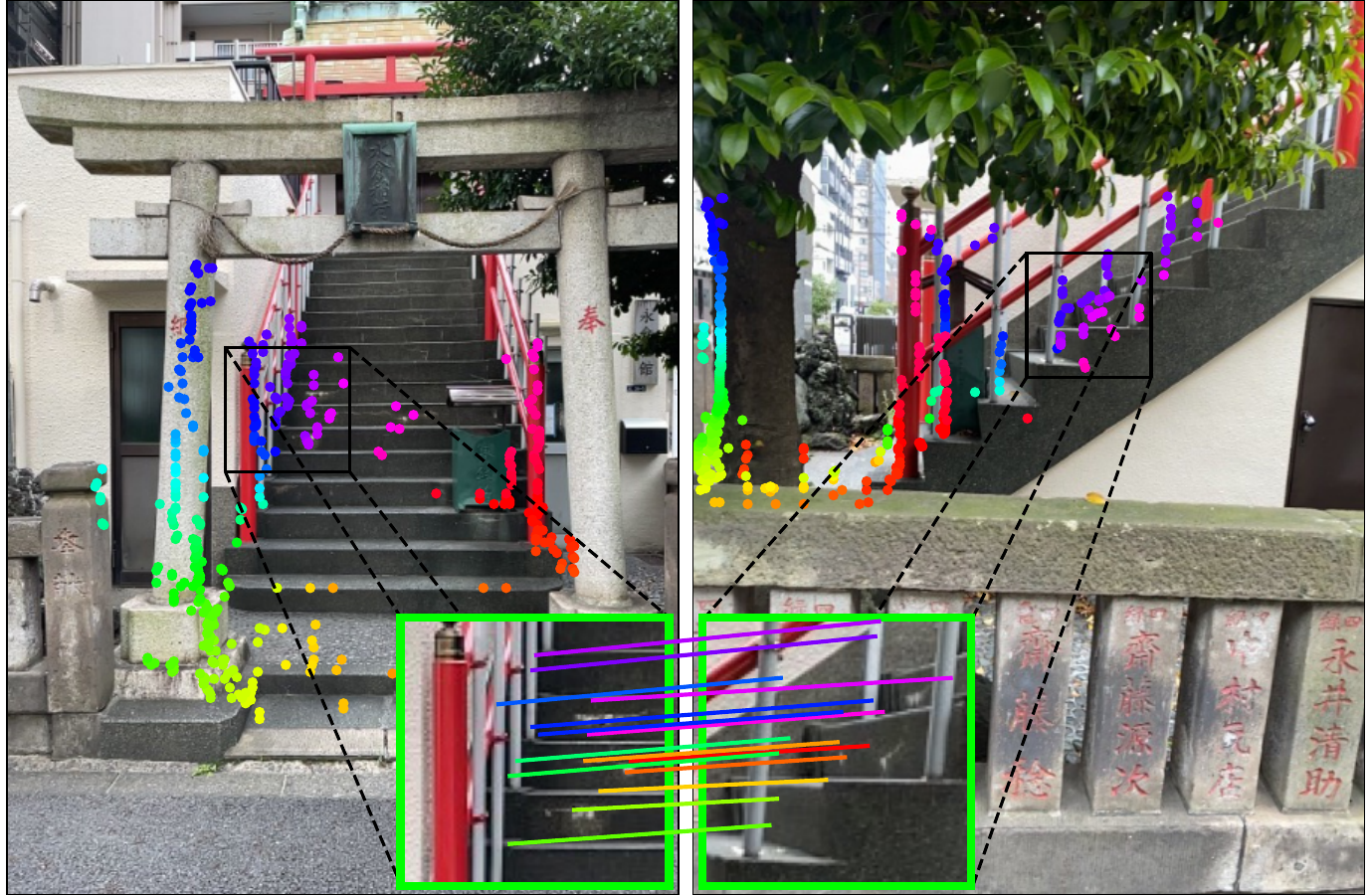}
    }
    \caption{
        Qualitative examples on the Map-free dataset.
        \textbf{Top row}: Pairs with strong viewpoint changes. Third one is a failure case.
        For clarity, we only draw a subset of all correspondences.
        \textbf{Bottom row}: We highlight interesting spots in close-up. These regions could hardly be matched by local keypoints. See text for details.
    }
    \label{fig:quali}
\end{figure*}

\subsection{Relative pose estimation}
\label{ssec:relpose}

\mypar{Datasets and protocol.}
Next, we evaluate for the task of relative pose estimation on the CO3Dv2~\cite{co3d} and RealEstate10k~\cite{realestate10K} datasets. CO3Dv2 contains 6 million frames extracted from approximately 37k videos, covering 51 MS-COCO categories. 
Ground-truth camera poses are obtained using COLMAP~\cite{colmap} from 200 frames in each video. 
RealEstate10k is an indoor/outdoor dataset that features 80K video clips on YouTube totalling 10 million frames, camera poses being obtained via SLAM with bundle adjustment. 
Following~\cite{posediffusion}, we evaluate \master{} on 41 categories from CO3Dv2 and 1.8K video clips from the test set of RealEstate10k. 
Each sequence is 10 frames long, we evaluate relative camera poses between all possible 45 pairs, not using ground-truth focals. %

\mypar{Baselines and metrics.} As before, matches obtained with \master{} are used to estimate Essential Matrices and relative pose. 
Please note that our predictions are always done pairwise, contrary to all other methods that leverage multiple views (at the exception of DUSt3R-PnP). 
We compare to recent data-driven approaches like RelPose~\cite{relpose}, RelPose++~\cite{relpose}, PoseReg and PoseDiff~\cite{posediffusion}, the recent RayDiff~\cite{raydiffusion} and \duster{}~\cite{dust3r}. 
We also report results for more traditional SfM methods like PixSFM~\cite{pixsfm} and COLMAP~\cite{colmapmvs} extended with SuperPoint~\cite{superpoint} and SuperGlue~\cite{superglue} (COLMAP+SPSG). 
Similar to~\cite{posediffusion}, we report the Relative Rotation Accuracy (RRA) and Relative Translation Accuracy (RTA) for each image pair to evaluate the relative pose error and select a threshold $\tau=15$ to report RTA@$15$ and RRA$@15$. Additionally, we calculate the mean Average Accuracy (mAA$30$), defined as the area under the accuracy curve of the angular differences at $min({\rm RRA}@30, {\rm RTA}@30)$.

\mypar{Results.} 
As shown in~\cref{tab:relpose_mvs}, SfM approaches tend to perform significantly worse on this task, mainly due to the poor visual support. 
This because images usually observe a small object, combined with the fact that many pairs have a wide baseline, sometimes up to 180$^\circ$. %
On the contrary, 3D grounded approaches like RayDiffusion, \duster{} and \master{} are the two most competitive methods on this dataset, the latter leading in translation and mAA on both datasets. Notably, on RealEstate our mAA score improves by at least $8.7$ points over the best multi-view methods and $15.2$ points over pairwise \duster{}. This showcases the accuracy and robustness of our approach to few input view setups. %

\begin{table*}[t]
    \begin{center}
    \small
    \caption{\emph{Left}: Multi-view pose regression on the CO3Dv2~\cite{co3d} and RealEstate10K~\cite{realestate10K} with 10 random frames. 
        Parenthesis () denote methods that do not report results on the 10 views set, we report their best for comparison (8 views). 
        We distinguish between (a) multi-view and (b) pairwise methods.
    \emph{Right:} Dense MVS results on the DTU dataset, in \emph{mm}. Handcrafted methods (c) perform worse than learning-based approaches (d) that train on this specific domain. Among the methods that operate in a zero-shot setting (e), \master{} is the only one attaining reasonable performance. 
    }
    \label{tab:relpose_mvs}
    \resizebox{0.53\textwidth}{!}{
    \begin{tabular}{llccccc}
    \toprule
    & \multirow{2}{*}{Methods}  & \multicolumn{3}{c}{Co3Dv2} &  & RealEstate10K\\ \cline{3-5} \cline{7-7} %
                           & \hspace{0.1pt} & RRA@15  & RTA@15 & mAA(30) &  & mAA(30)       \\ 
    \midrule
    \multirow{8}{*}{(a)} & Colmap+SG~\cite{superpoint, superglue} & 36.1    & 27.3   & 25.3    &  & 45.2          \\
    &PixSfM~\cite{pixsfm}                  & 33.7    & 32.9   & 30.1    &  & 49.4          \\
    &RelPose~\cite{relpose}              & 57.1    & -      & -       &  & -             \\
    &PosReg~\cite{posediffusion}           & 53.2    & 49.1   & 45.0    &  & -             \\
    &PoseDiff ~\cite{posediffusion}   & 80.5    & 79.8   & 66.5    &  & 48.0          \\
    &RelPose++~\cite{relposepp}           & (85.5)    & -      & -       &  & -             \\
    &RayDiff~\cite{raydiffusion}   & (93.3)    & -   & -    &  & -          \\ %
    &\duster-GA~\cite{dust3r}             & \textbf{96.2}    & 86.8   & 76.7    &  & 67.7          \\ \midrule
    \multirow{2}{*}{(b)}&\duster~\cite{dust3r}               & 94.3    & 88.4  & 77.2    &  & 61.2          \\
    &\bf{\master{}}            & 94.6   & \bf{91.9}   &  \bf{81.8} &  & \bf{76.4} \\ 
    \bottomrule
    \end{tabular}
}
\quad
\resizebox{0.32\linewidth}{!}{
    \begin{tabular}{llccc}
    \toprule
    & Methods &  Acc.$\downarrow$ & Comp.$\downarrow$ & Overall$\downarrow$       \\
    \midrule
    \multirow{4}{*}{(c)} & Camp~\cite{camp} &  0.835 & 0.554 & 0.695 \\
    &Furu~\cite{furu} & 0.613 & 0.941 & 0.777 \\
    &Tola~\cite{tola} & 0.342 & 1.190 & 0.766 \\
    &Gipuma~\cite{gipuma} & \textbf{ 0.283} & 0.873 & 0.578\\
    \midrule
    \multirow{7}{*}{(d)} &MVSNet~\cite{mvsnet} & 0.396 & 0.527 & 0.462 \\
    &CVP-MVSNet~\cite{cvp-mvsnet} & 0.296 & 0.406 & 0.351 \\
    &UCS-Net~\cite{ucs-net} &  0.338 & 0.349 & 0.344 \\
    & CER-MVS~\cite{cermvs}  & 0.359 & 0.305 & 0.332 \\
    & CIDER~\cite{cider} &  0.417 & 0.437 & 0.427 \\
    &PatchmatchNet~\cite{pathcmatchnet}  & 0.427 & 0.277 & 0.352 \\
    & GeoMVSNet~\cite{geomvsnet}  & 0.331 & \textbf{0.259} & \textbf{0.295} \\
    \midrule
    \multirow{2}{*}{(e)}& \duster~\cite{dust3r}  &   2.677  &  0.805  & 1.741  \\ %
    & \master &   0.403   &  0.344  &  0.374 \\ %
    \bottomrule
    \end{tabular}
}
\normalsize
\end{center}
\end{table*}

\subsection{Visual localization}
\label{ssec:visloc}

\mypar{Datasets.}
We then evaluate \master{} for the task of absolute pose estimation on the Aachen Day-Night\cite{aachen} and InLoc\cite{inloc} datasets. Aachen comprises 4,328 reference images taken with hand-held cameras, %
as well as 824 daytime and 98 nighttime query images taken with mobile phones in the old inner city of Aachen, Germany. 
InLoc\cite{inloc} is an indoor dataset with challenging appearance variation between the 9,972 RGB-D + 6DOF pose database images and the 329 query images taken from an iPhone 7.

\mypar{Metrics.} We report report the percentage of successfully localized images within three thresholds: (0.25m, 2°), (0.5m, 5°) and (5m, 10°) for Aachen and 
(0.25m, 10°), (0.5m, 10°), (1m, 10°) for InLoc.

\mypar{Results} are reported in Table \ref{tab:visloc}. We study the performance of \master{} with variable number of retrieved images. As expected, a greater number of retrieved images (top40) yields better performance, achieving competitive performance on Aachen and significantly outperforming the state of the art on InLoc. 
Interestingly, our approach still performs very well even with a single retrieved image (top1), showcasing the robustness of 3D grounded matching. 
We also include direct regression results, which are rather poor, showing a striking impact of the dataset scale on the localization error, \ie small scenes are much less affected (see results on Map-free in \ref{ssec:mapfree}).
This confirms the importance of feature matching to estimate reliable poses.

\begin{table*}[t]
    \caption{Visual localization results on Aachen Day-Night and InLoc. We report our results for different number of retrieved database images (topN).
    }
    \label{tab:visloc}
    \begin{center}
    \small
    \resizebox{.9\linewidth}{!}{
    \begin{tabular}{lccccc}
    \toprule
    \multirow{2}{*}{Methods}  & \multicolumn{2}{c}{AachenDayNight\cite{aachen}}& & \multicolumn{2}{c}{InLoc\cite{inloc}}\\ \cline{2-3} \cline{5-6} %
                            \hspace{0.1pt} & Day& Night& & DUC1& DUC2\\ \midrule
    Kapture+R2D2 \cite{kapture2020}             & 	 \textbf{91.3/97.0/99.5} & \textbf{78.5/91.6/100} &  & 		41.4/60.1/73.7    & 	47.3/67.2/73.3          \\
    SP+SuperGlue \cite{superglue}             & 89.8/{96.1}/99.4 & 77.0/90.6/\textbf{100} &  & 	49.0/68.7/80.8    & 53.4/77.1/82.4          \\ 
    SP+LightGlue \cite{lightglue}             & {90.2}/96.0/99.4 & 77.0/91.1/\textbf{100} &  & 	49.0/68.2/79.3    & 55.0/74.8/79.4          \\ 
    LoFTR \cite{sun2021loftr}             & 88.7/95.6/99.0 & \textbf{78.5}/90.6/99.0 &  & 	47.5/72.2/84.8    & 54.2/74.8/85.5          \\ 
    
    DKM \cite{edstedt2023dkm}             & - & - &  & 	51.5/75.3/86.9    & 63.4/82.4/87.8          \\
    
    \midrule

    \duster~top1~\cite{dust3r}             & 72.7/89.6/98.1 & 59.7/80.1/93.2 &  & 	36.4/55.1/66.7    & 27.5/42.7/49.6          \\ 
    \duster~top20~\cite{dust3r}             & 79.4/94.3/\textbf{99.5} & 74.9/91.1/99.0 &  & 	53.0/74.2/89.9    & 	61.8/77.1/84.0          \\ \midrule
    \bf{\master{}~top1}            & 79.6/93.5/98.7    & 70.2/88.0/97.4 &  &  41.9/64.1/73.2  & 38.9/55.7/62.6 \\ 
    \bf{\master{}~top20}            & 83.4/95.3/99.4    & 76.4/\textbf{91.6/100} &  &  55.1/77.8/90.4 & \textbf{71.0}/84.7/89.3\\ 
    \bf{\master{}~top40}            & 82.2/93.9/\textbf{99.5}    & 75.4/\textbf{91.6/100} &  &  \bf 56.1/\textbf{79.3/90.9} & \textbf{71.0/87.0/91.6} \\ 
    \bf{\master{}~direct reg.~top1}            & 1.5/4.5/60.7    & 1.6/4.2/47.6 &  & 13.1/32.3/58.1 & 	10.7/26.0/38.2 \\
    \bottomrule
    \end{tabular}}
    \normalsize
    \end{center}
\end{table*}

\subsection{Multiview 3D reconstruction}
\label{ssec:mvs}
We finally perform MVS by triangulating the obtained matches. Note that the matching is performed in full resolution without prior knowledge of cameras, and the latter are only used to triangulate matches in ground-truth reference frame. 
We remove spurious 3D points via geometric consistency post-processing~\cite{pathcmatchnet}. 

\mypar{Datasets and metrics.}
We evaluate our predictions on the DTU~\cite{dtu} dataset. Contrary to all competing learning methods, we apply our network in a zero-shot setting, \ie we do not train nor finetune on the DTU train set and apply our model as is.
In \cref{tab:relpose_mvs} we report the average accuracy, completeness and Chamfer distances error metrics as provided by the authors of the benchmarks. 
The accuracy for a point of the reconstructed shape is defined as the smallest Euclidean distance to the ground-truth, and
the completeness of a point of the ground-truth as the
smallest Euclidean distance to the reconstructed shape. 
The overall Chamfer distance is the average of both previous metrics. 

\mypar{Results.} Data-driven approaches trained on this domain significantly outperform handcrafted ones, cutting the Chamfer error by half. 
To the best of our knowledge, we are the first to draw such conclusion in a zero-shot setting. 
\master{} not only outperforms the \duster{} baseline but also compete with the best methods, all without leveraging camera calibration nor poses for matching, neither having seen this camera setup before.

\section{Conclusion}
Grounding image matching in 3D with \master{} significantly raised the bar on camera pose and localization tasks on many public benchmarks. 
We successfully improved \duster{} with matching, getting the best of both worlds: enhanced robustness, while attaining and even surpassing what could be done with pixel matching alone. 
We introduced a fast reciprocal matcher and a coarse to fine approach for efficient processing, allowing users to balance between accuracy and speed. 
\master{} is able to perform in few-view regimes (even in top1), that we believe will greatly increase versatility of localization. 

\clearpage
\LARGE
\chapter{\textbf{Appendix}
\vspace{1cm}
\normalsize

\appendix
\input{appendix}

{
    \small
    \bibliographystyle{ieeenat_fullname}
    \bibliography{main}
}

\end{document}

%% file: appendix.tex
In this appendix, we first present additional qualitative examples on various tasks in~\cref{sec:quali},
followed by  a proof of convergence of the fast reciprocal matching algorithm and an in-depth study of the related performance gains in~\cref{sec:reciprocal}. 
We finally show an ablative study concerning the impact of \emph{coarse-to-fine} matching in~\cref{sec:coarse-to-fine-abl}. 

\begin{figure}[h!]
    \centering
    \includegraphics[width=.95\linewidth]{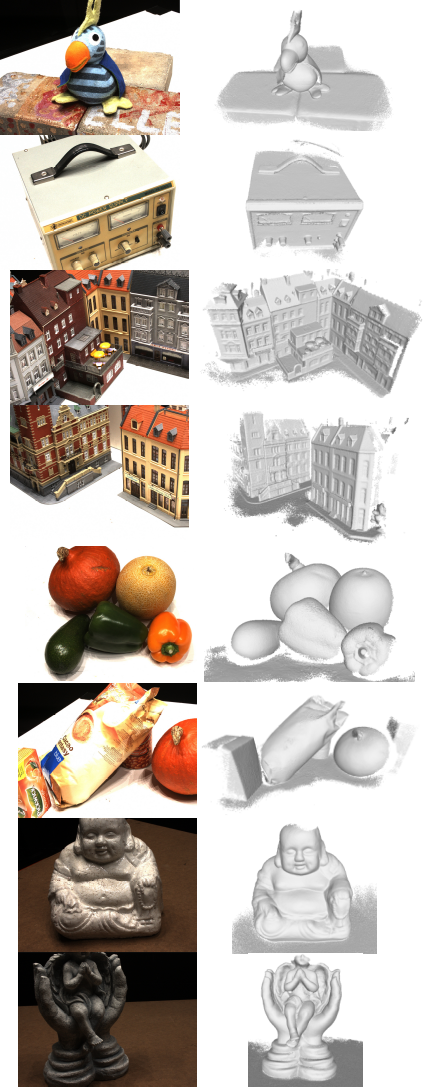}
    \caption{
        Qualitative MVS results on the DTU dataset~\cite{dtu} simply obtained by triangulating the dense matches from \master. 
    }
    \label{fig:quali_dtu}
    \vspace{-5mm}
\end{figure}

\begin{figure*}[h!]
    \centering
    \includegraphics[width=0.49\linewidth]{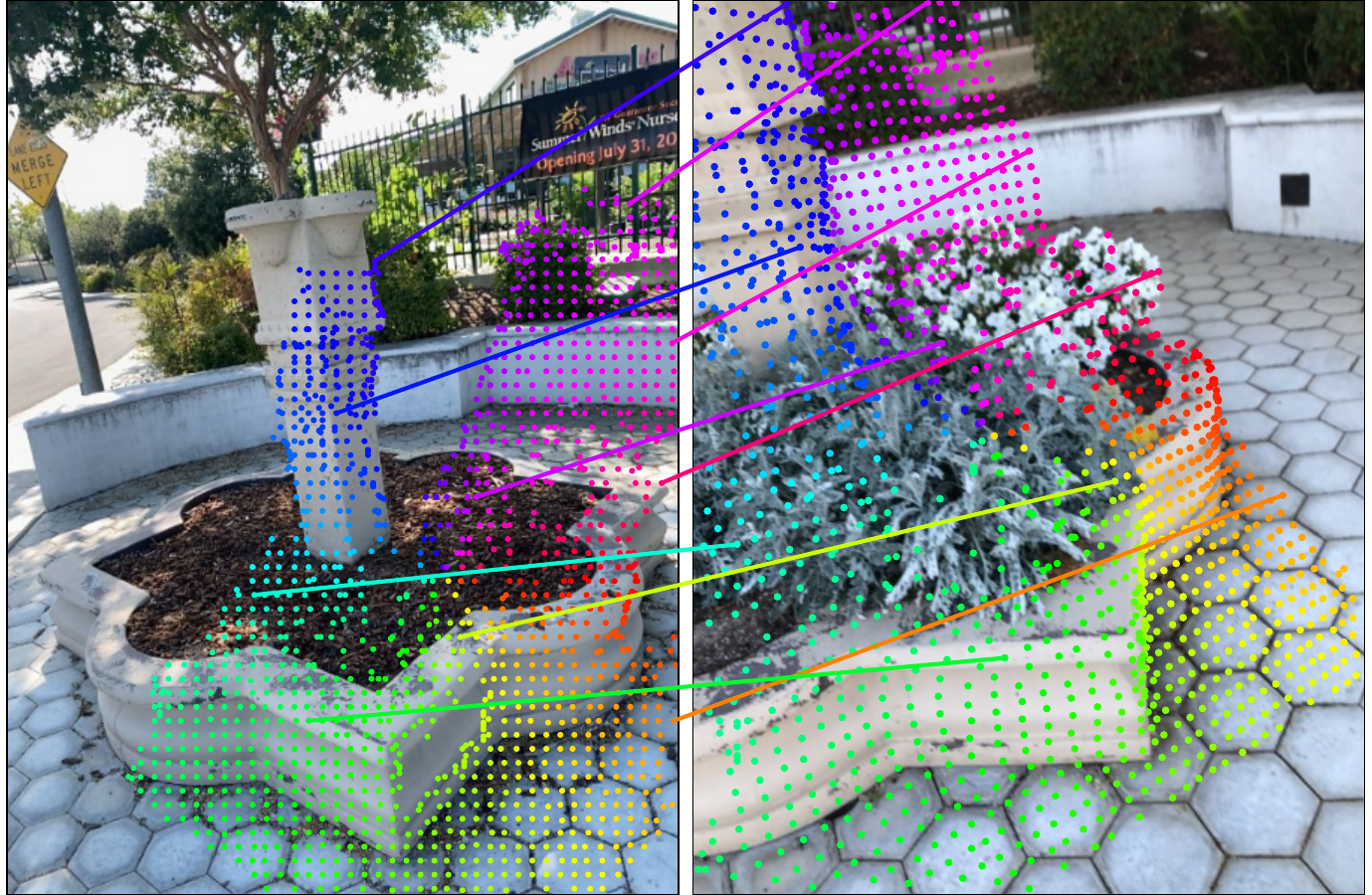}
    \includegraphics[width=0.49\linewidth]{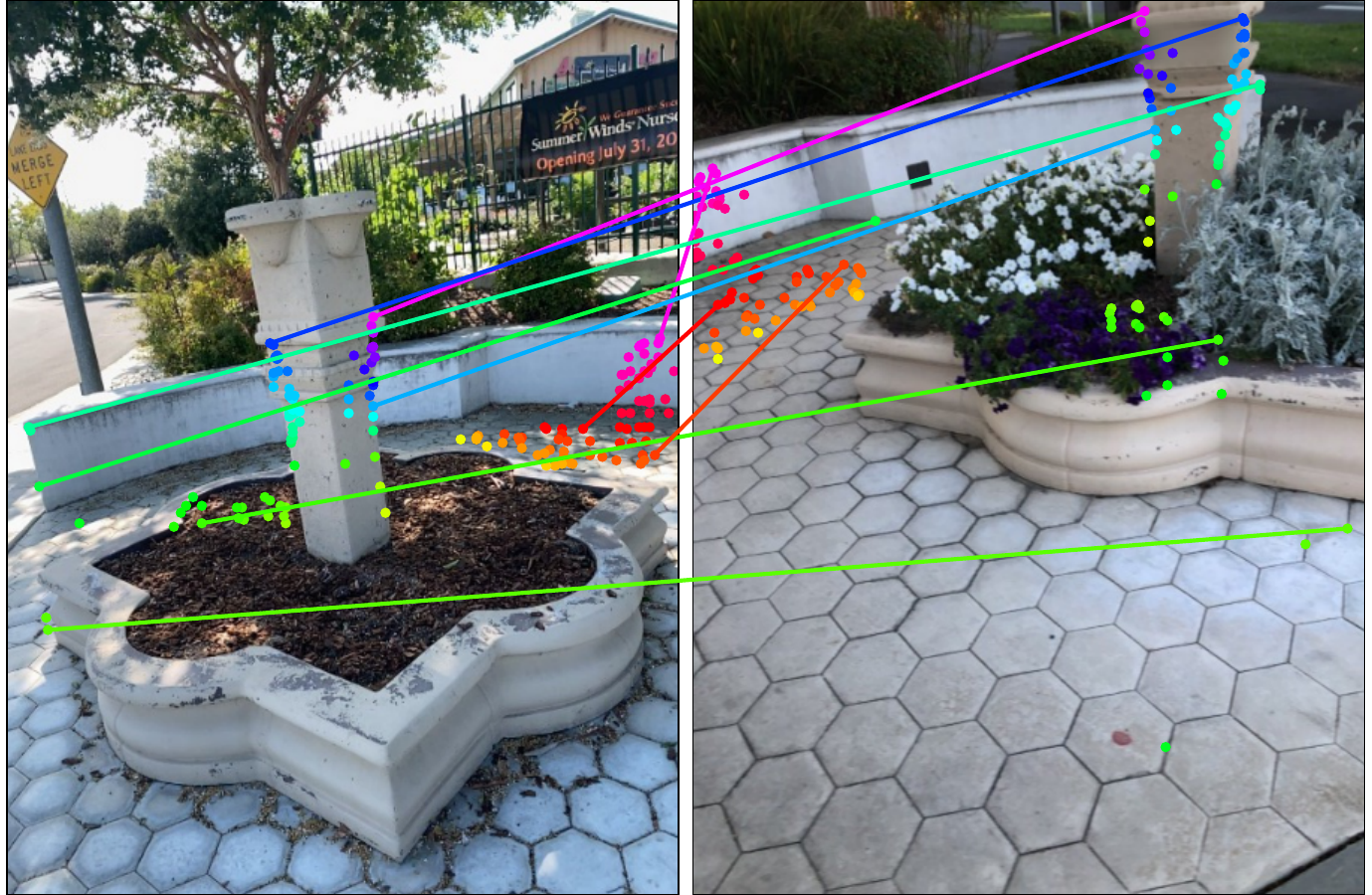}
    \includegraphics[width=0.49\linewidth]{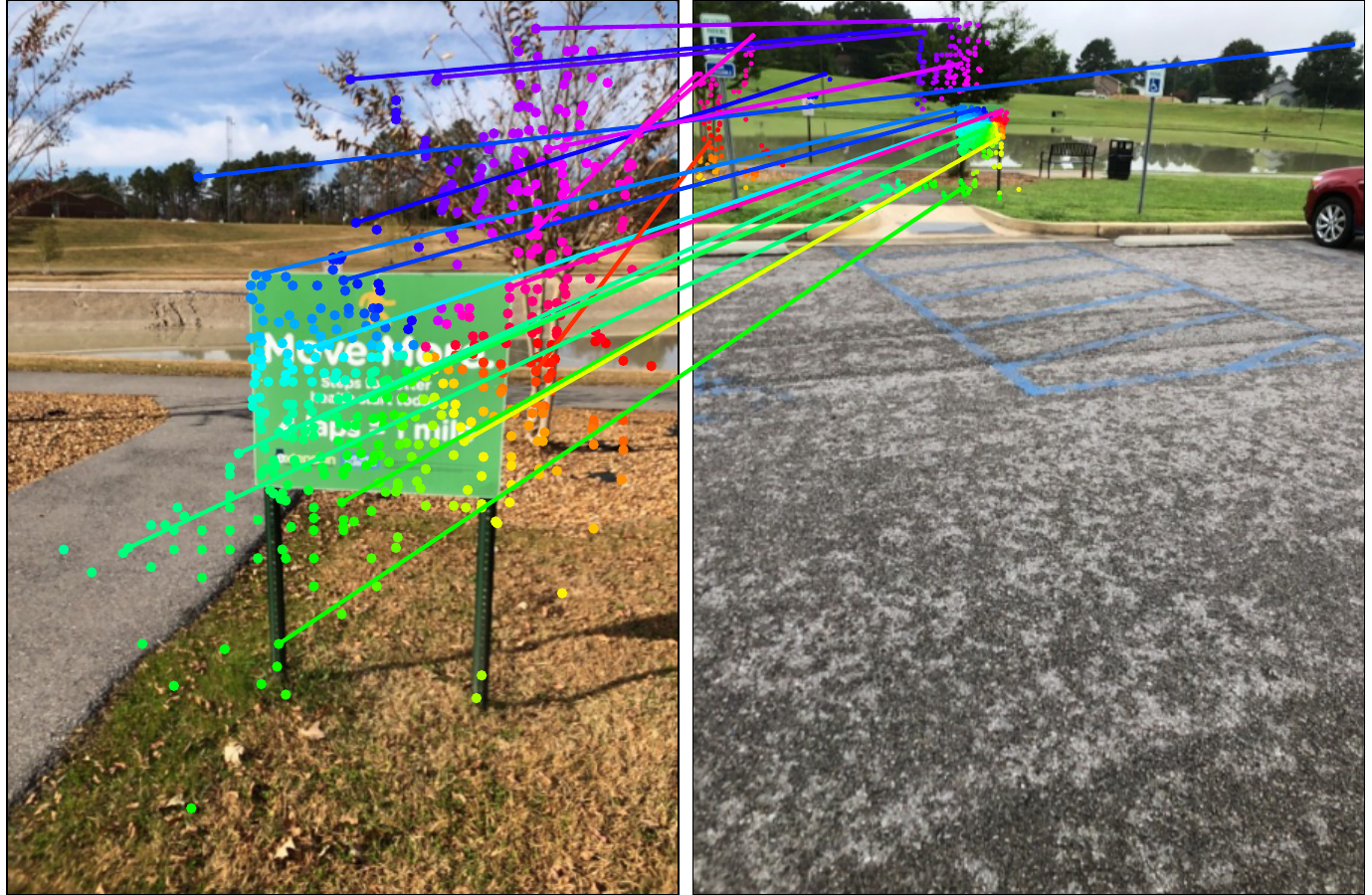}
    \includegraphics[width=0.49\linewidth]{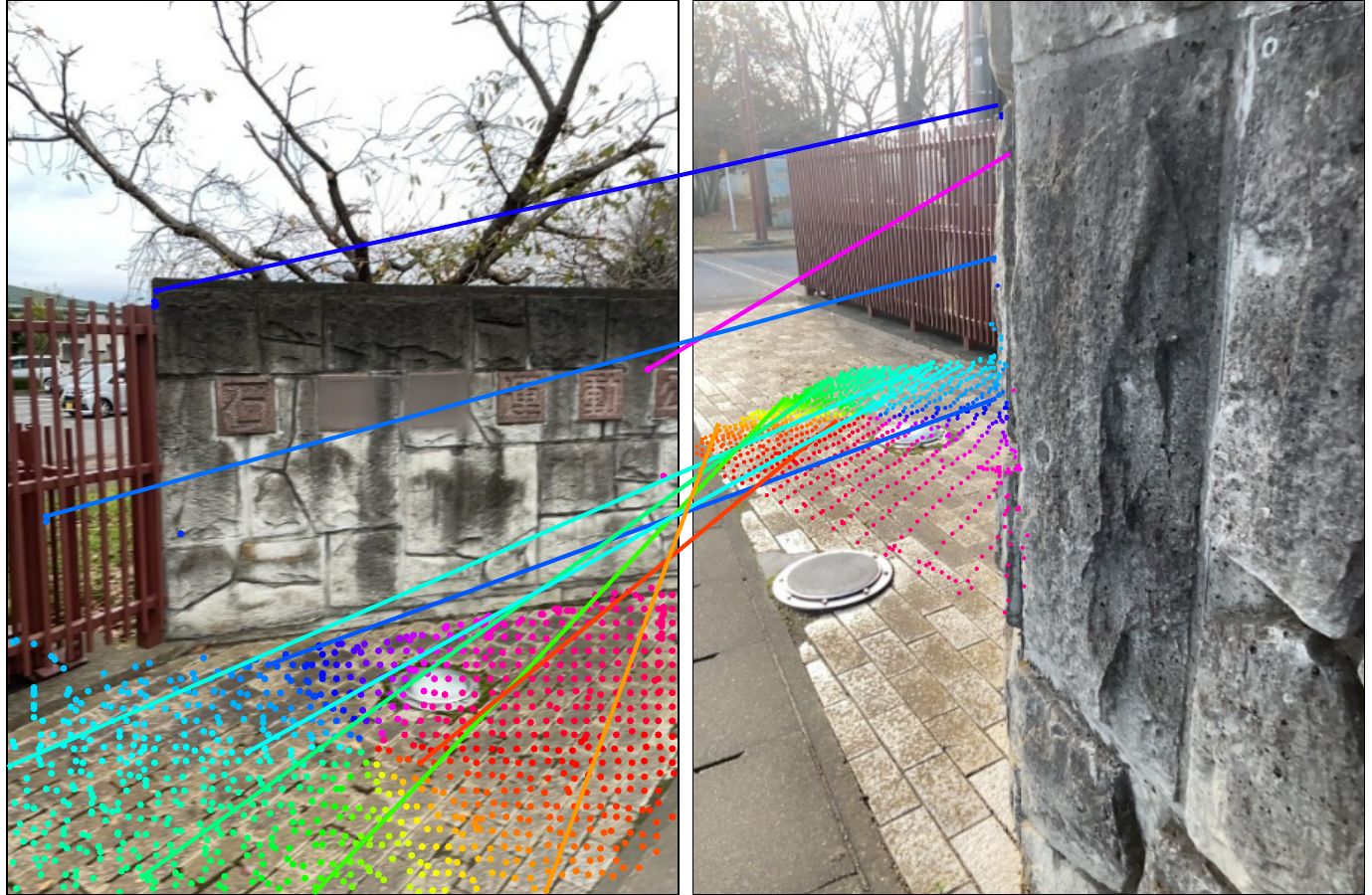}
    \caption{
        \text{Qualitative examples of matching on Map-free localization benchmark.} 
    }
    \label{fig:quali_mapfree}
\end{figure*}

\section{Additional Qualitative Results}
\label{sec:quali}

We provide here additional qualitative results on the DTU~\cite{dtu}, InLoc~\cite{inloc}, Aachen Day-Night datasets~\cite{aachen} and the Map-free benchmark~\cite{mapfree}. 

\mypar{MVS on DTU.}
We show in ~\cref{fig:quali_dtu} the output point clouds after post-processing, shaded with approximate normals from the tangent planes based on the $50$ nearest neighbors. We wish to emphasize again that the point clouds are raw values obtained via triangulation of the \emph{coarse-to-fine} matches of \master{}. The matching was performed in an one-versus-all strategy, meaning that we did not leverage the epipolar constraints coming from the GT cameras, which is in stark contrast with all existing approaches for MVS. \master{} is particularly precise and robust, giving sharp and dense details. The reconstructions are complete even in low-contrast homogeneous regions like the surfaces of the vegetables or the sides of the power supply. The matching is also robust to varied textures or materials, and also to violations of the Lambertian assumption, \ie specularities on the vegetables, plastic surfaces or the white sculpture.

\begin{figure*}[h!]
    \centering
    \begin{tabular}{cc}
     \includegraphics[width=0.45\linewidth]{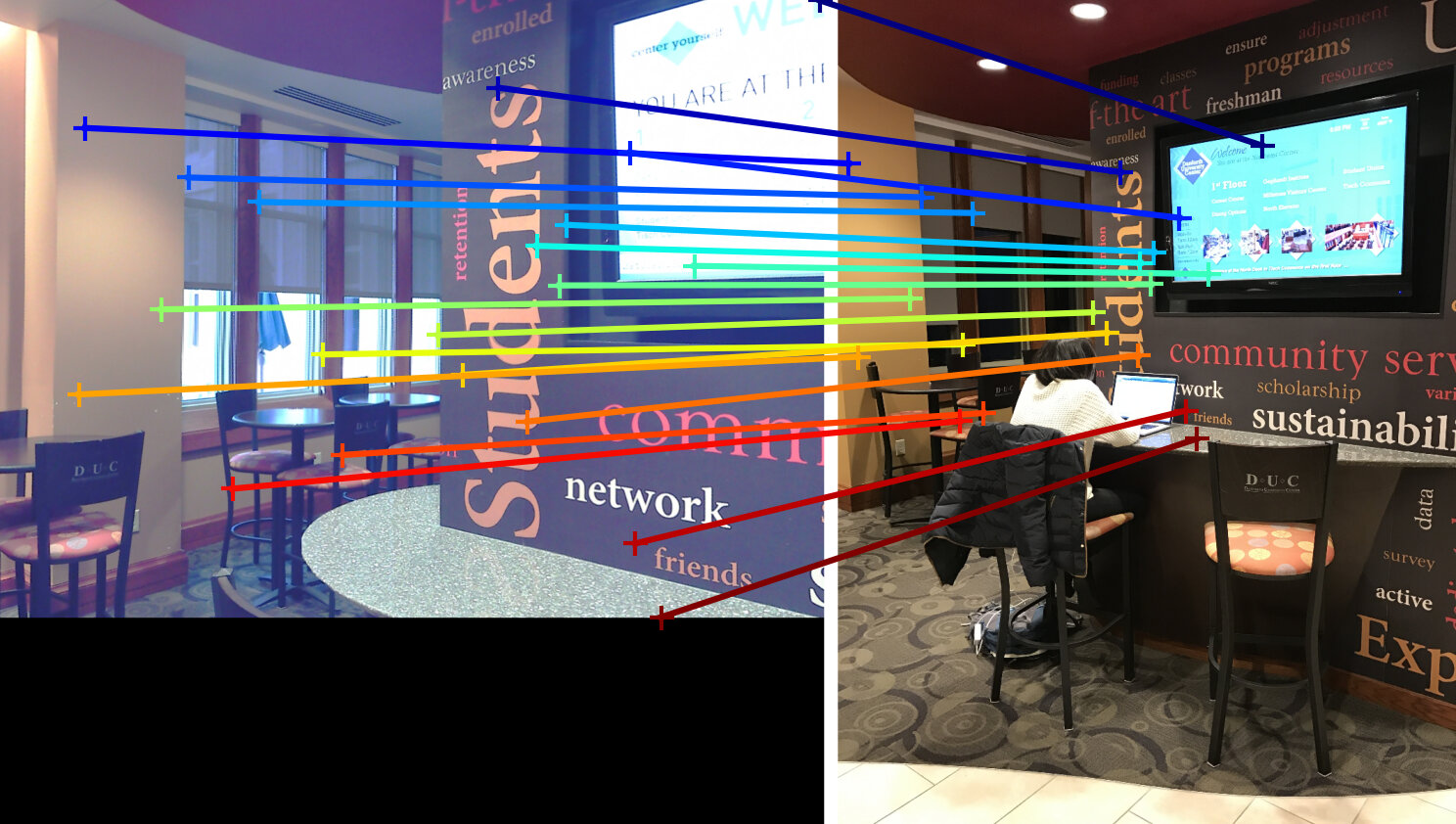} &   \includegraphics[width=0.45\linewidth]{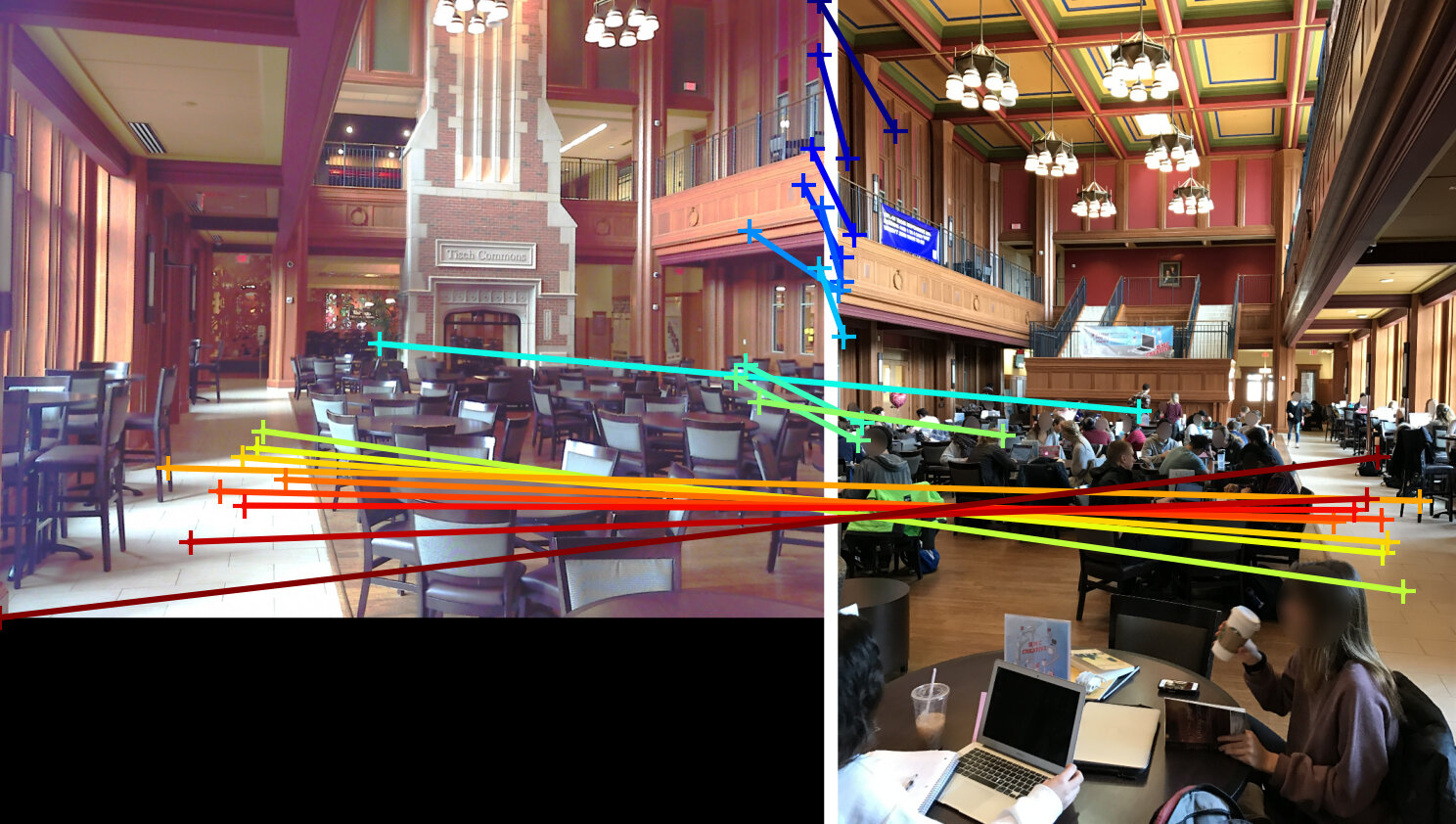} \\
     \includegraphics[width=0.45\linewidth]{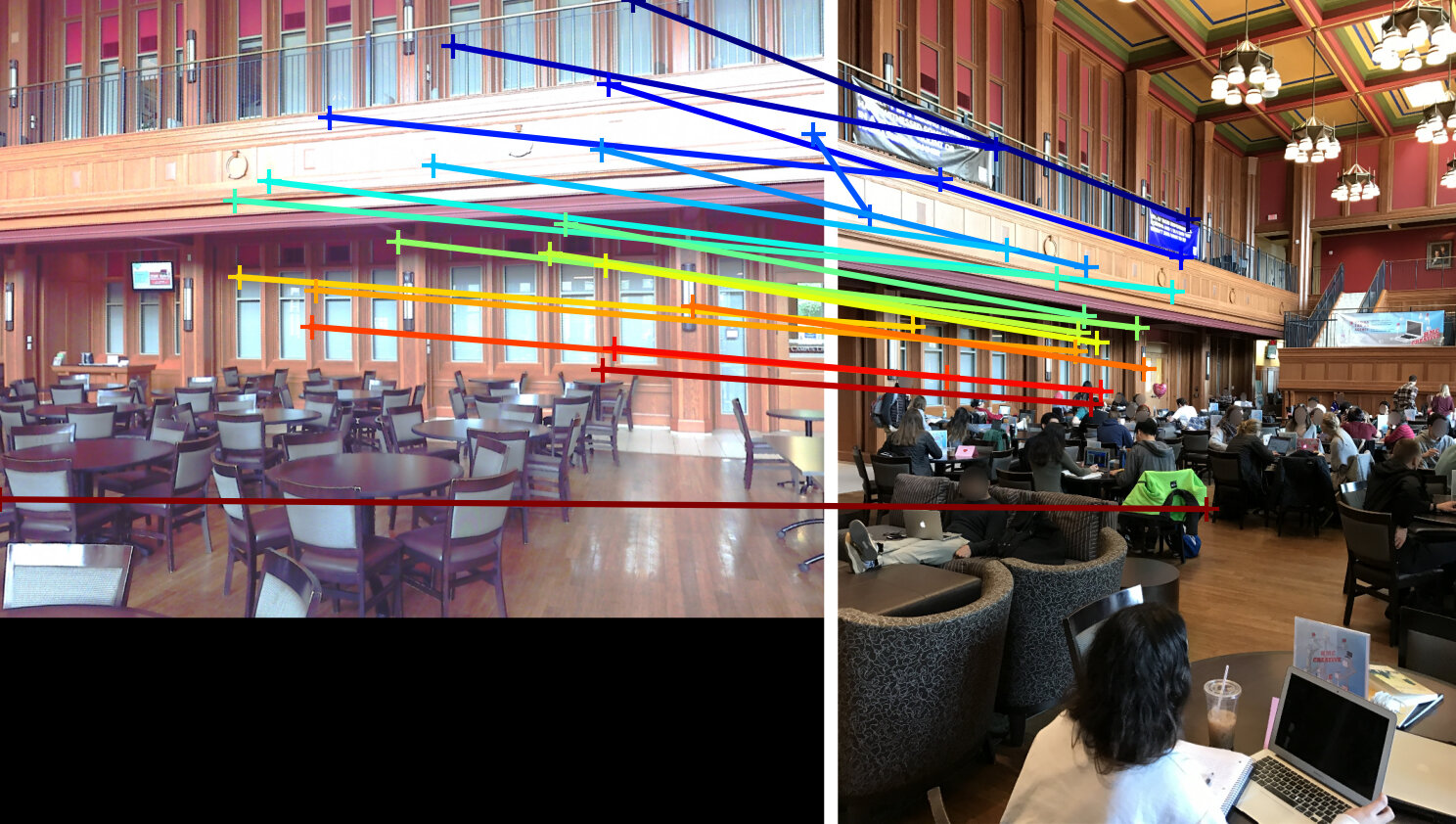} &   \includegraphics[width=0.45\linewidth]{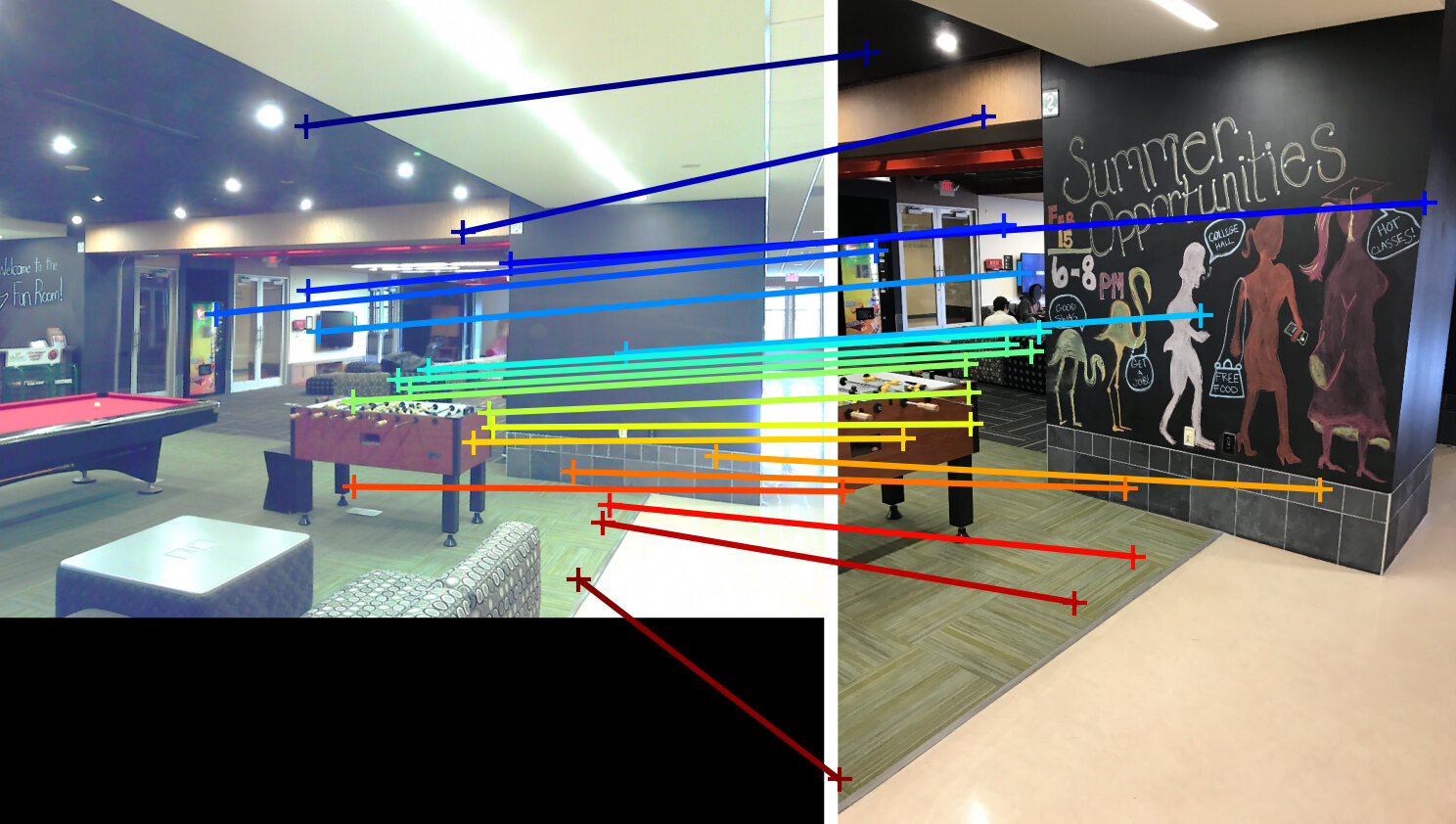} \\
     \includegraphics[width=0.45\linewidth]{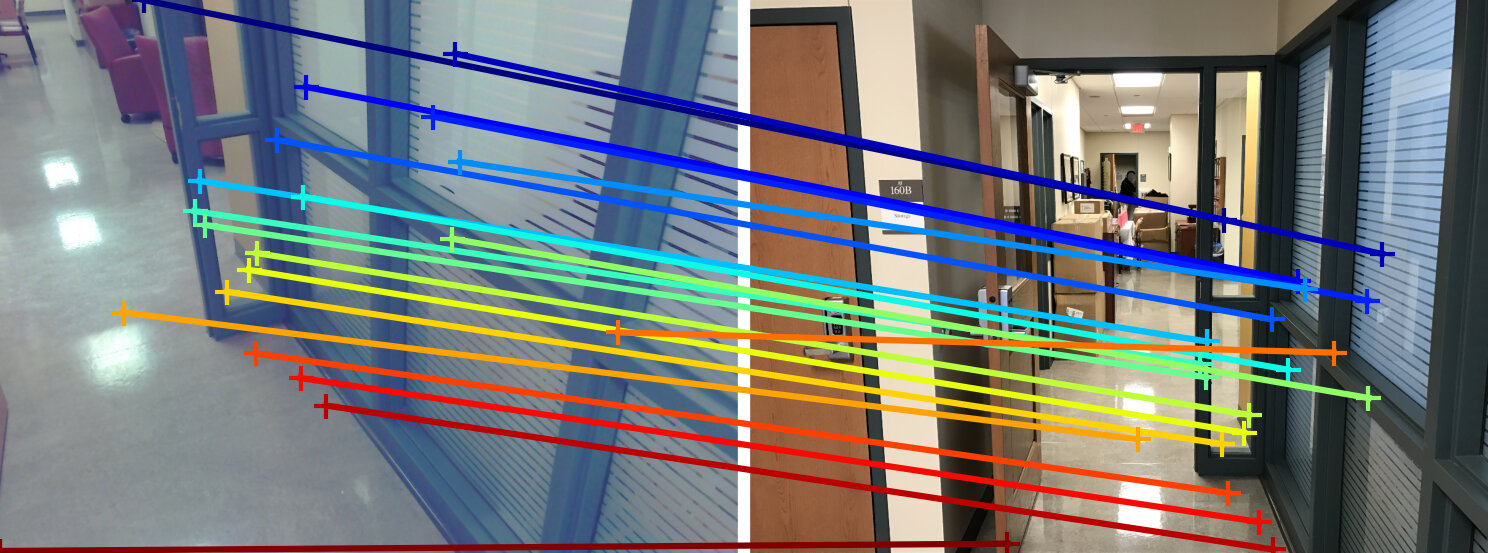} &   \includegraphics[width=0.45\linewidth]{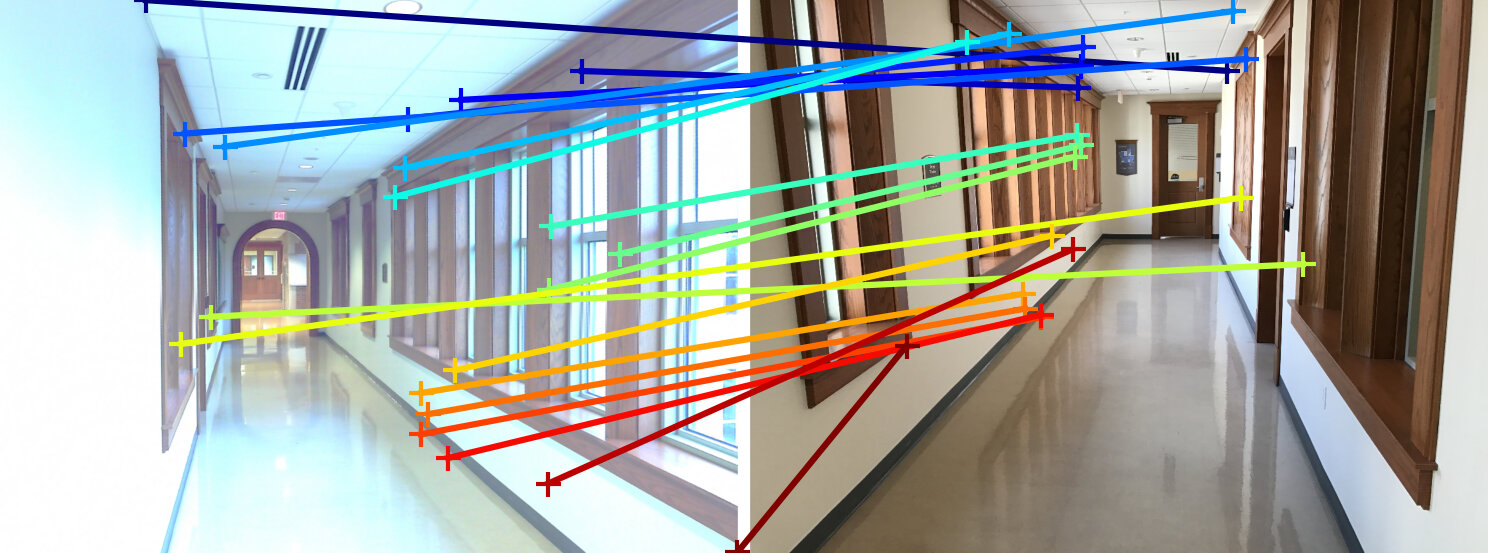} \\
    \end{tabular}
    \caption{
        \text{Qualitative examples of matching on the InLoc localization benchmark.}
    }
    \label{fig:quali_inloc}
    \vspace{-.5cm}
\end{figure*}
\begin{figure*}[h!]
    \centering
    \begin{tabular}{cc}
     \includegraphics[width=0.58\linewidth]{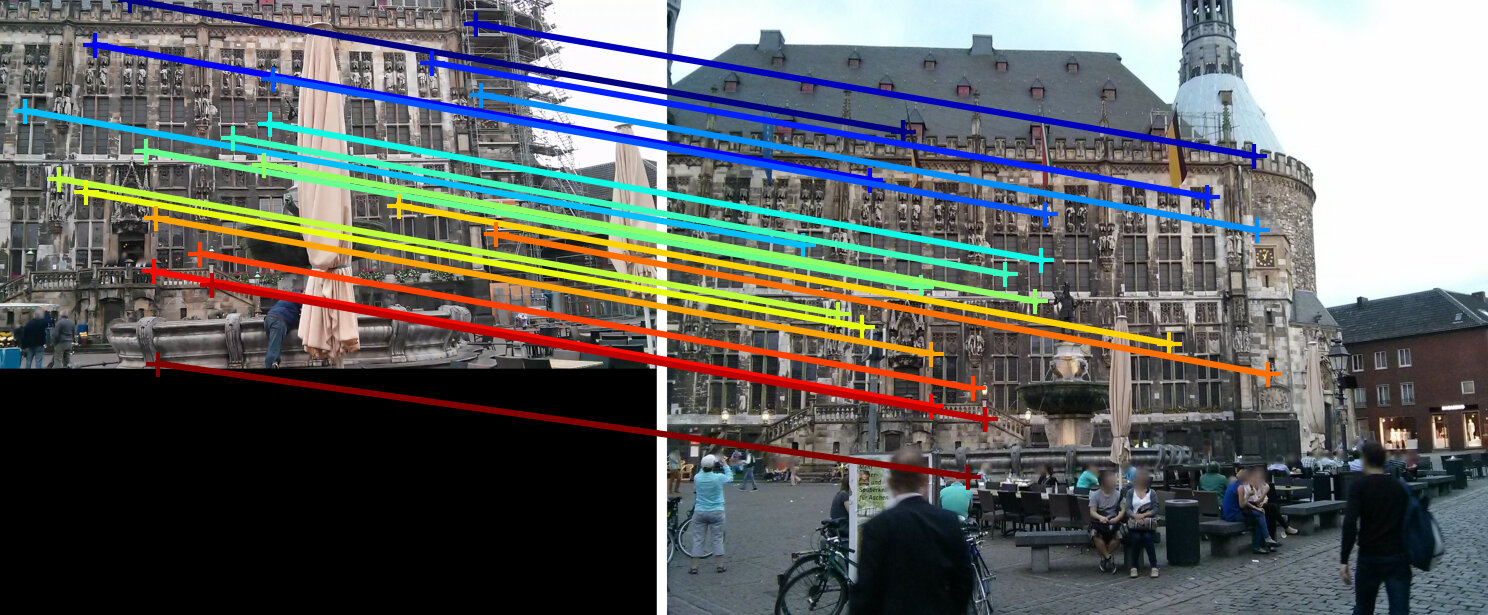} &   \includegraphics[width=0.42\linewidth]{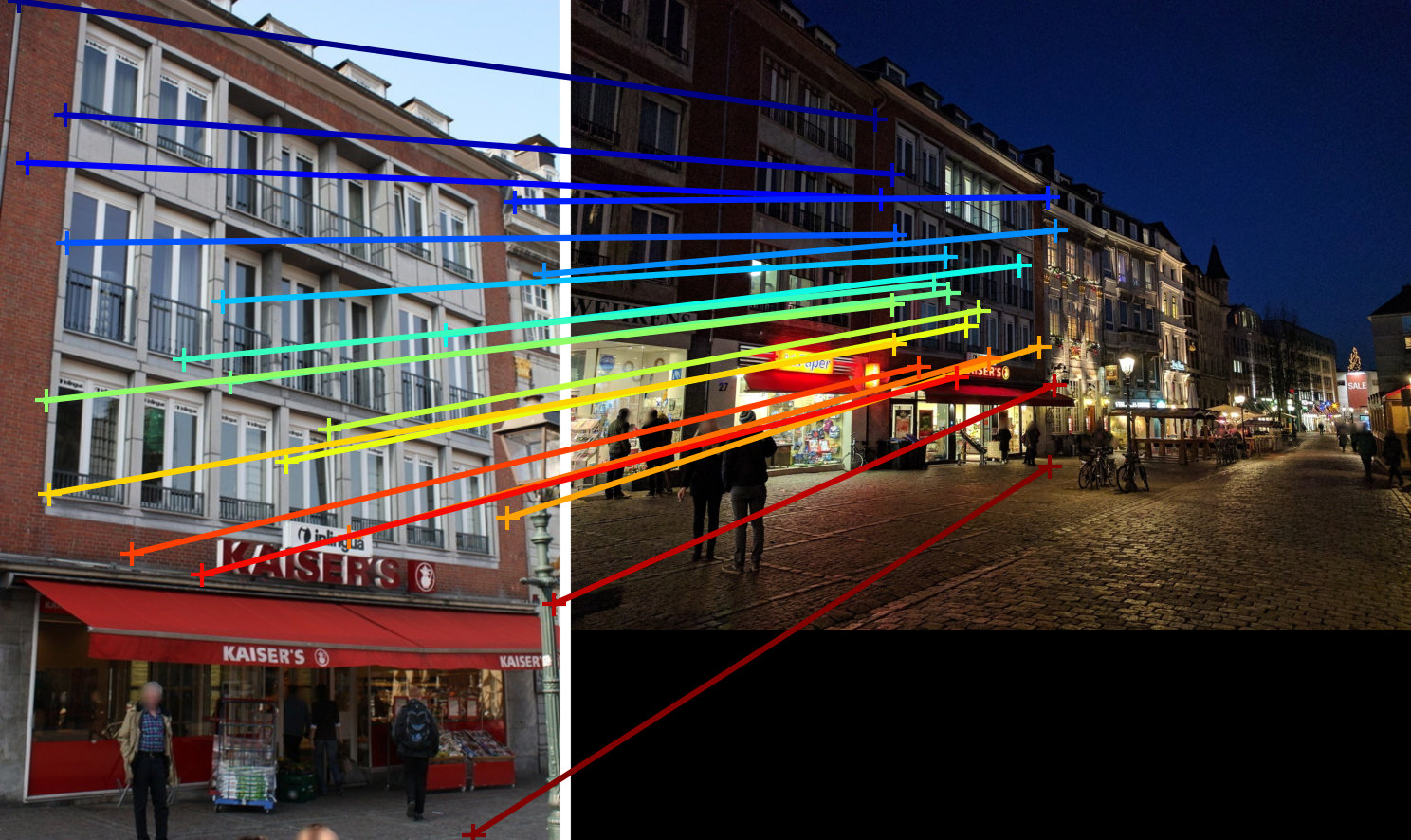} \\
     \includegraphics[width=0.58\linewidth]{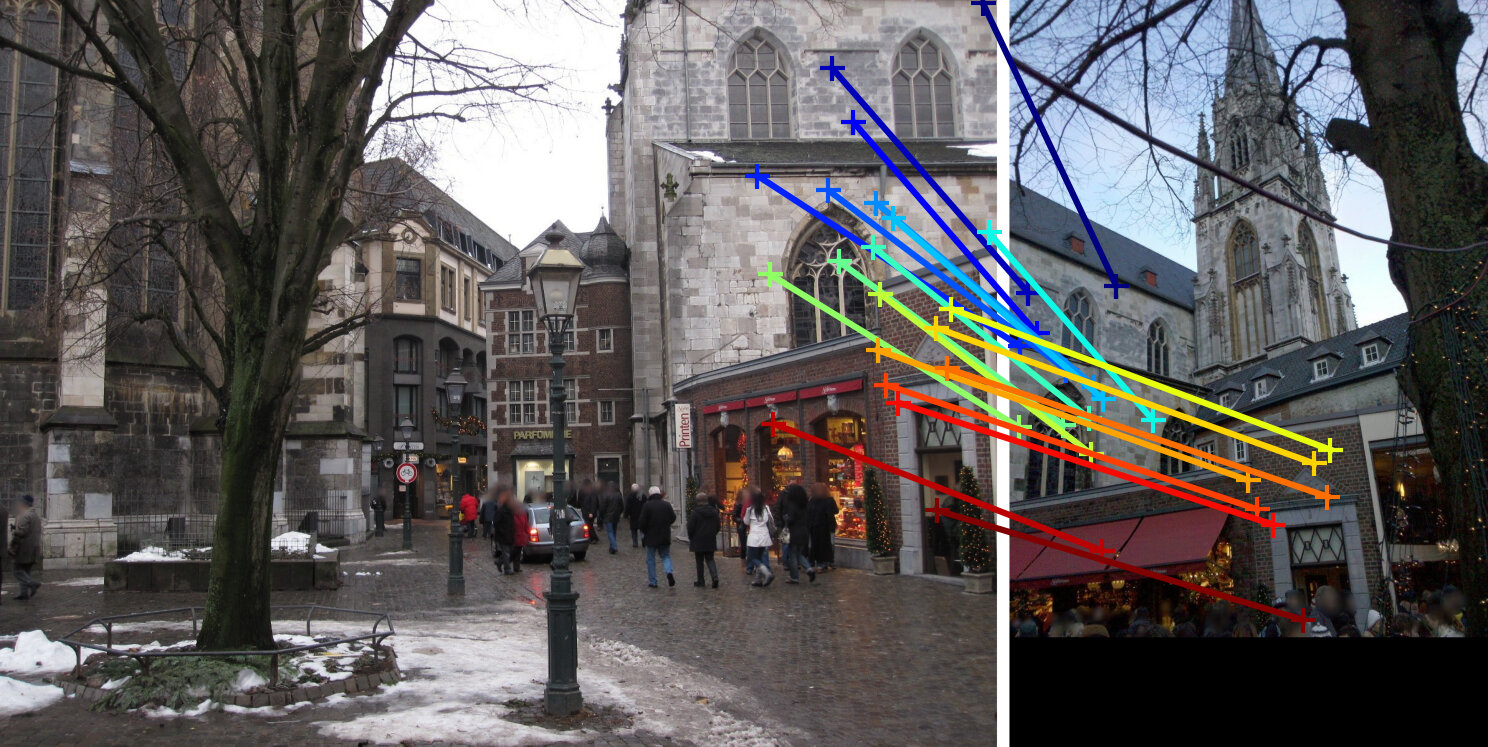} &   \includegraphics[width=0.42\linewidth]{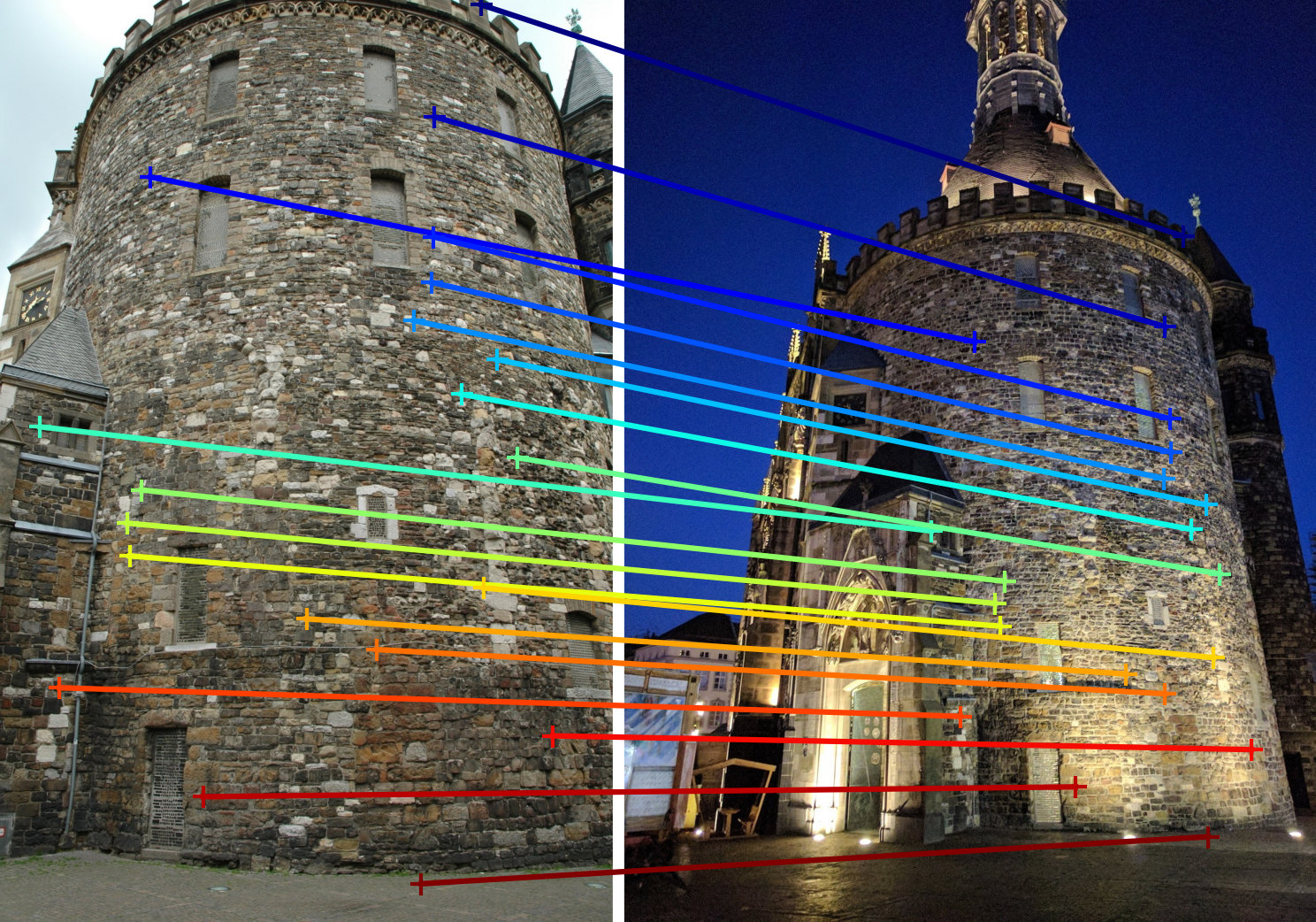} \\
     \includegraphics[width=0.58\linewidth]{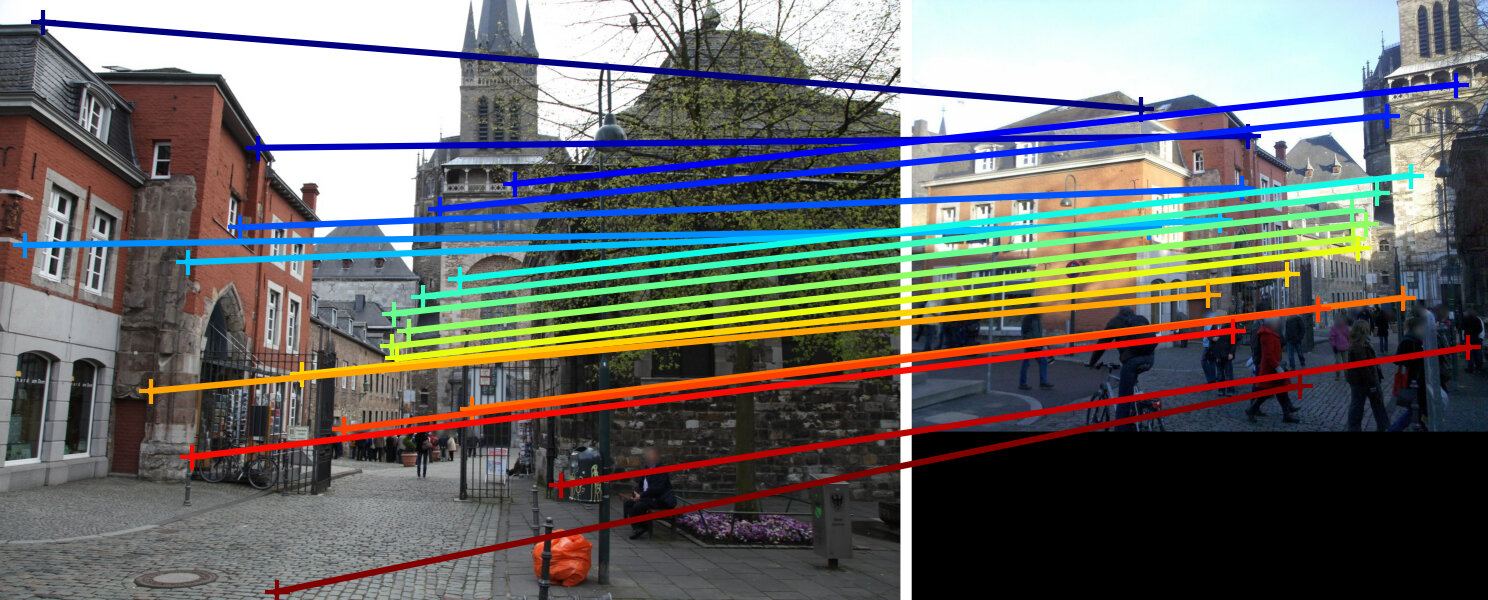} &   \includegraphics[width=0.42\linewidth]{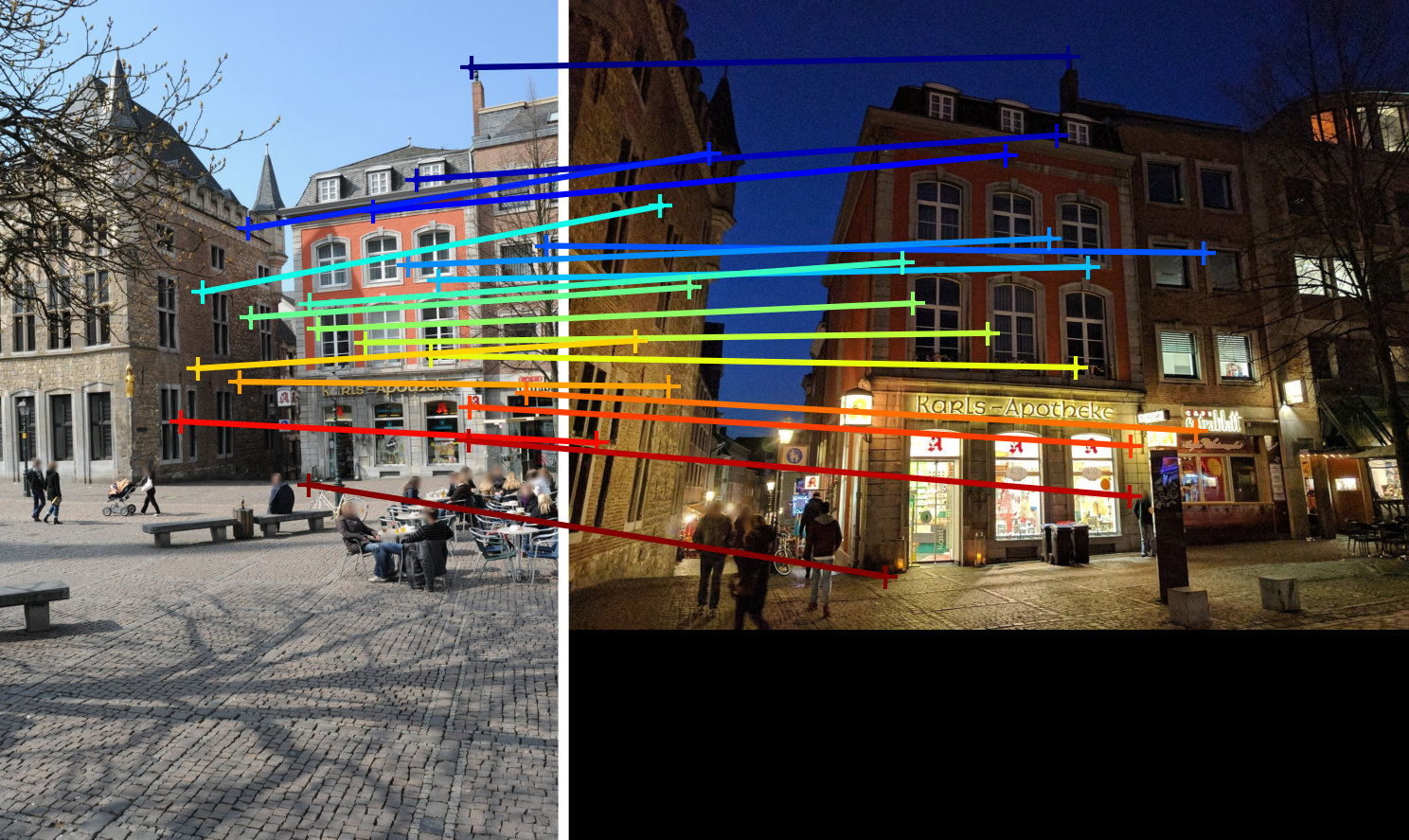} \\
    \end{tabular}
    \caption{
        \text{Qualitative examples of matching on the Aachen Day-Night localization benchmark.} Pairs from the day subset are on the left column, and pairs from the night subset are on the right column.
    }
    \label{fig:quali_aachen}
\end{figure*}

\mypar{Qualitative matching results.}
We show a few examples of matches 
\cref{fig:quali_mapfree} for the Map-free benchmark~\cite{mapfree},
in \cref{fig:quali_inloc} for the InLoc~\cite{inloc} dataset 
and 
in \cref{fig:quali_aachen} for the Aachen Day-Night dataset~\cite{aachen}.
The proposed \master{} approach is robust to extreme viewpoint changes, and still provides approximately correct correspondences in such cases (right-hand side pairs of Map-free in \cref{fig:quali_mapfree}), even for views facing each other (coffee tables or corridor pairs of InLoc~\ref{fig:quali_inloc}). 
This is reminiscent of the capabilities of \duster{} that provided an unprecedented robustness to such cases. 
Similarly, our approach handles large scale differences (\eg on Map-free in~\cref{fig:quali_mapfree}) repetitive and ambiguous patterns, as well as environmental and day/night illuminations changes (\cref{fig:quali_aachen}). 
Interestingly, the accuracy of correspondences output by \master{} gracefully degrades when the viewpoint baseline increases. Even in extreme cases where correspondences get very coarsely estimated, approximately correct relative camera poses can still be recovered.
Thanks to these capabilities, \master{} reach state-of-the-art performance or close to it on several benchmarks in a zero-shot setting. We hope this work will foster research in the direction of pointmap regression for a multitude of vision tasks, where robustness and accuracy are critical.

\begin{figure*}[h!]
    \centering
    \includegraphics[width=.7\linewidth]{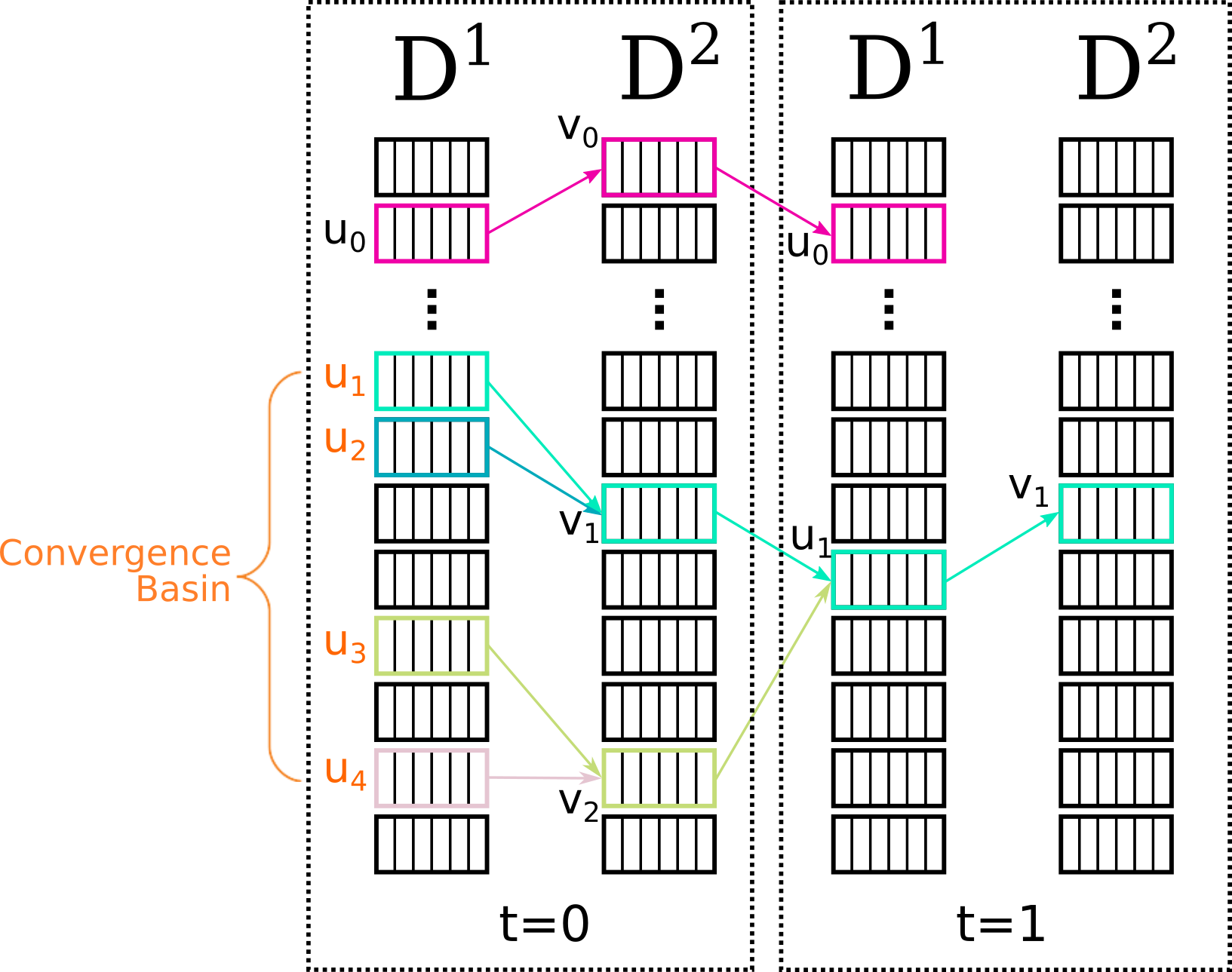}
     \caption{
        \textbf{Illustration of the iterative FRM algorithm.} Starting from 5 pixels in $I^1$ at $t=0$, the FRM connects them to their Nearest Neighbors (NN) in $I^2$, and maps them back to their NN in $I^1$. If they go back to their starting point (top pink), a cycle (reciprocal match) is detected and returned. Otherwise (bottom) the algorithm continues iterating until a cycle is detected for all starting samples, or until the maximal number of iterations is reached. We show in orange the starting points of a \emph{\orange{convergence basin}}, \ie nodes of a sub-graph for which the algorithm will converge towards the same cycle. For clarity, all edges of $\mathcal{G}$ were not drawn.
    }
    \label{fig:graph}
\end{figure*}

\section{Fast Reciprocal Matching}
\label{sec:reciprocal}

\subsection{Theoretical study}

We detail here the theoretical proofs of convergence of the Fast Reciprocal Matching algorithm presented in Sec.3.3 of the main paper. Contrary to the traditional bipartite graph matching formulation~\cite{graphmatch}, where the complete graph is used for the matching, we wish to decrease the computational complexity by calculating only a smaller portion of it. 
As explained in equation (14) of the main paper, considering the two predicted sets of features $D^1$, $D^2 \in \mathbb{R}^{H \times W \times d}$, partial reciprocal matching boils down to finding a subset of the reciprocal correspondences, \ie mutual Nearest Neighbors (NN):

\begin{align}
    \M = \{(i,j) \ | \  j=\N{2}(\D{1}_i) \text{ and } i=\N{1}(\D{2}_j) \}, \label{eq:corres2}\\
    \text{with } \N{A}(\D{B}_j) = \argmin_{i} \left\Vert \D{A}_i - \D{B}_j \right\Vert. \label{eq:nn2}
\end{align}

We remind here the behavior of the algorithm: an initial set of $k$ pixels of $I^1$, $U^0 = \{ U^0_n\}^k_{n=1}$  with $ k \ll WH$, is mapped to their NN in $I^2$, yielding $V^1$, that are then mapped to their nearest neighbors back to $I^1$:
\begin{equation}
    U^t \longmapsto [ \N{2}(\D{1}_u) ]_{u \in U^t} \equiv V^t \longmapsto [ \N{1}(\D{2}_v) ]_{v \in V^t} \equiv U^{t+1}
\end{equation}
After this back-and-forth mapping, the reciprocal matches (\ie those which form a cycle) are recovered and removed from $U^{t+1}$. The remaining "active" ones are mapped back to $I^2$ and reciprocity is checked again. We iterate this process for a few iterations. After enough iterations we discard any active sample remaining.

It is important to note that the NN algorithm we use is deterministic and consistently returns the same index in the case where multiple descriptors in the other image share the same minimal distance (or maximal similarity), although this is very unlikely since descriptors are real-valued.

\mypar{Proof of Convergence.} 
By design, Fast Reciprocal Matching (FRM) operates on the directed bipartite graph $\mathcal{G}$ of nearest neighbors between $I^1$ and $I^2$. $\mathcal{G}$ contains oriented edges $\mathcal{E}$. %
All nodes, \ie pixels, belong to $\mathcal{G}$ since we add an edge for each pixel's nearest neighbor, but note that all pixels cannot reach all other pixels. For example, two reciprocal pixels in $I^1$ and $I^2$ are only connected to each other and to no other pixels. This means $\mathcal{G}$ is composed of possibly multiple disjoint sub-graphs $\mathcal{G}^i, 1 \le i \le HW $ with directed edges $\mathcal{E}^i$ (see~\cref{fig:graph}). 

\begin{proposition}
\label{prop:single_cycle}
There can be only one cycle in each sub-graph $\mathcal{G}^i$.
\end{proposition}

\begin{proof}
This is a rather trivial fact, since we build $\mathcal{G}$ s.t. only one edge exits each node. If one were to follow the path of a sub-graph $\mathcal{G}^i$, once a node that belongs to a cycle is reached, no edge can exit the cycle, for the only exiting edge is already part of the cycle. A second cycle (or more) thus cannot exist in $\mathcal{G}^i$.     
\end{proof}

\begin{lemma}
Each of the subgraph $\mathcal{G}^i$ is either a single cycle or a special arborescence, \ie a directed graph where, from any node there exist a single path towards a root cycle. %
\end{lemma}

\begin{figure*}[t]
    \centering
    \includegraphics[width=.8\linewidth]{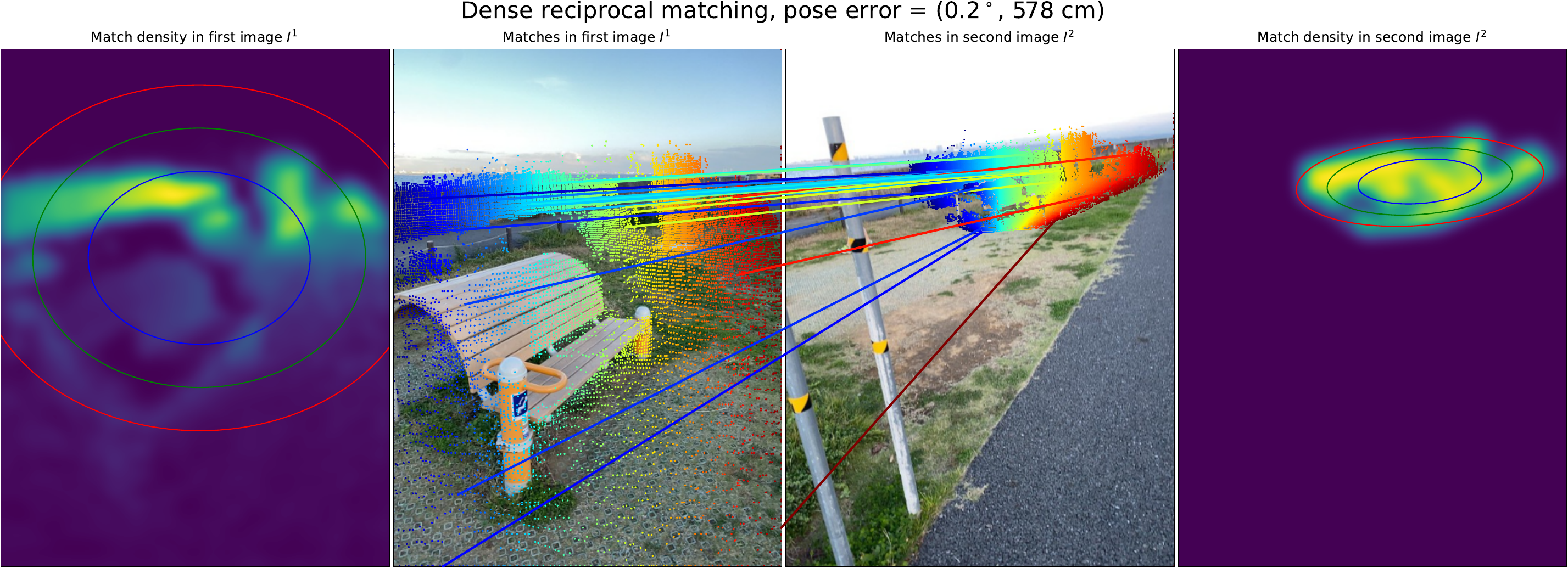}
    \\[2mm]
    \includegraphics[width=.8\linewidth]{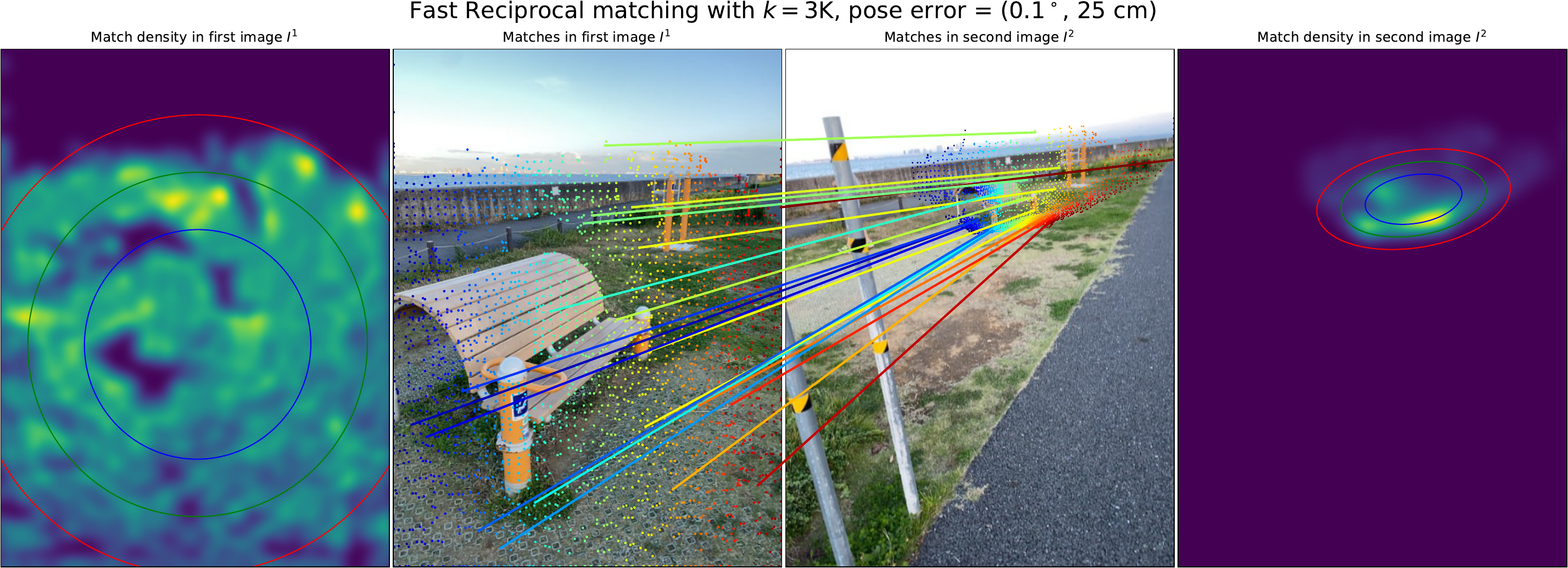}
    \\[-2mm]
    \caption{
        Illustration of the difference in matching density when using dense reciprocal matching (baseline) and fast reciprocal matching with $k=3000$. 
        Fast reciprocal matching samples correspondences with a bias for large \orange{convergence basins}, resulting in a more uniform coverage of the images. 
        Coverage can be measured in terms of the mean and standard deviation $\sigma$ of the point matches in each density map, plotted as colored ellipses (red, green and blue correspond respectively to $1\sigma, 1.5\sigma$ and $2\sigma$).}
    \label{fig:matching_density}
\end{figure*}

\begin{figure*}[h]
    \centering
    \includegraphics[width=0.7\linewidth]{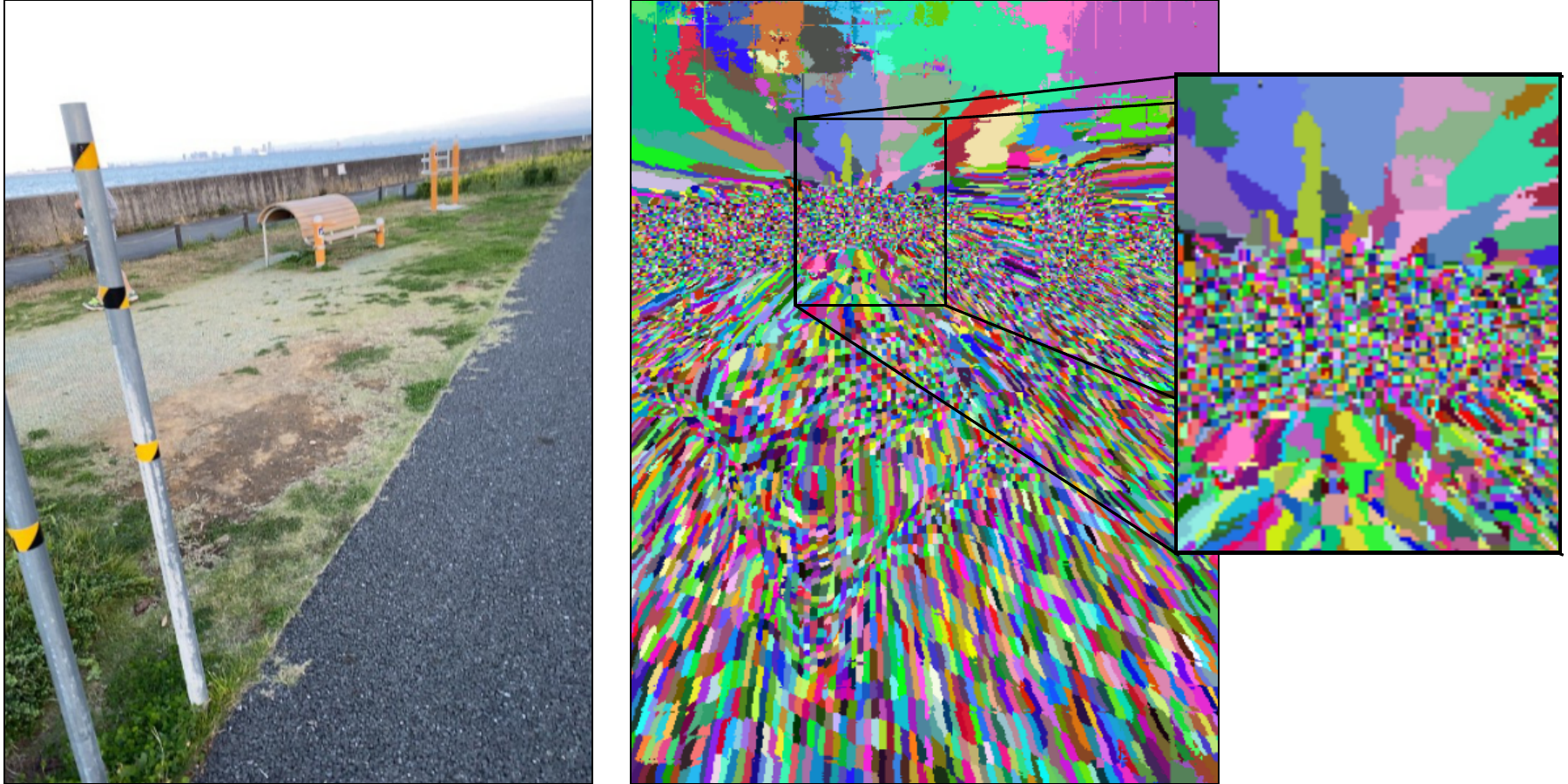}
    \caption{Illustration of \orange{convergence basins} for one of the image in \cref{fig:matching_density}.
        Each basin is filled with the same (random) color.
        A convergence basin is an area for which any of its point will converge to the same correspondence when applying the fast reciprocal matching algorithm. }
    \label{fig:basins}
\end{figure*}

\begin{proof}
The former follows naturally from the previous explanation: since there can only be a single cycle in $\mathcal{G}^i$, it can naturally be a cycle. We now demonstrate the latter, \ie when $\mathcal{G}^i$ is not trivially a cycle. Let us march on $\mathcal{G}^i$ starting from an arbitrary node $a$, to which is attached a descriptor $D^1_a$. 
The only edge exiting this node goes to its nearest neighbor $NN_2(D^1_a) = b$. Now at node $b$, we do the same and follow the only edge exiting back to $I^1$: $NN_1(D^2_b) = c$. Alternating between $I^1$ and $I^2$, we get $NN_2(D^1_c)=d$, $NN_1(D^2_d)=e$ and so forth. We denote $s(u,v) = D_u^{1\top} D^2_v$ the similarity score of an edge between two nodes $u$ and $v$, $(u,v) \in \mathcal{E}^i$. Because edges are nearest neighbors, we note that $s(a,b) \le s(c,b)$. This trivially stems from the fact that if $s(c,b) < s(a,b)$ then the nearest neighbor of $b$ would no longer be $c$ but at least $a$. Expanding this property to the path along $\mathcal{G}^i$ it follows that:

\begin{equation}
\label{eq:simincreas}
   s(a,b) \le s(c,b) \le s(c,d) \le s(e,d) ...  
\end{equation}

Meaning that the similarity score monotonously increases as we walk along the graph. There is a finite number of nodes in $\mathcal{G}^i$ so this sequence reaches the upper-bound similarity value $s(u,v)$. Because $s(u,v)$ is the maximal similarity in $\mathcal{G}^i$, this ensures that $NN_2(D^1_u) = v$ and $NN_1(D^2_v) = u$ forming a cycle of at least two nodes. This means there is always a cycle in $\mathcal{G}^i$, between the maximal similarity pair. Following~\cref{prop:single_cycle}, we can conclude that there is no other cycle in $\mathcal{G}^i$ and that each starting point is thus guaranteed to lead towards the root via a single path, forming an arborescence with a cycle at its root. 
\end{proof}
Note that the root cycle can be of more than two nodes if more than one greatest similarity of~\cref{eq:simincreas} are perfectly equal and the NN algorithm creates a greater cycle. Because $\mathcal{G}$ is a bipartite graph, $\mathcal{G}^i$ is also bipartite, meaning the end-cycle is composed of an even number of nodes.
In practice however, we work with floating-point descriptors of dimension $24$. For greater cycles to exist, \eg cycles of $4$ nodes $a$, $b$, $c$, $d$, the similarities must satisfy increasingly prohibitive constraints, \eg $s(a,b) = s(c,b) = s(c,d) = s(a,d)$. 
This is extremely unlikely with real-valued distance and we consider it is negligible. 

\begin{corollary}
Regardless of the starting point in $\mathcal{G}^i$, the FRM algorithm always converges towards reciprocal matches.
\end{corollary}

This follows naturally from the above: we did not make any assumption about the starting point of this walk nor about the sub-graph it belongs to. For any starting point in the graph, \ie for all initial pixels $U$, the FRM algorithm will by design follow the sub-graph of nearest neighbors that will ultimately lead to the root cycle, which is by definition a reciprocal match.

 We illustrate this behavior in~\cref{fig:graph}. In the upper part (pink) the starting point $u_0$ directly lies in a cycle containing two nodes $u_0$ and $v_0$ and the algorithm stops after the first cycle verification at step $t=1$.
The bottom part shows a more complex case of \orange{convergence basin}, where several starting points $u_1$, $u_2$, $u_3$, $u_4$ lead to resp. two nodes $v_1$ and $v_2$ in $I^2$. Following the path to the root of the arborescence, and updating $U$ and $V$ along the way, the algorithm finds a cycle between $u_1$ and $v_1$ at timestep $t=1$. From $5$ initial pixel positions, the algorithm returned a unique reciprocal correspondence.

Note that it is possible to artificially build a graph that maximizes the number of NN queries thus impacting the computational efficiency, but these are very unlikely in practice as seen in Figure 2 (center) of the main paper. The number of active samples, \eg samples that did not reach a cycle, quickly drops to $0$ after only $6$ iterations, leading to a significant speed-up in computation (right).

\begin{proposition}
Starting from $k \ll HW$ samples, the FRM algorithm recovers a subset $\mathcal{M}_k$ of all possible reciprocal correspondences of cardinality $|\mathcal{M}_k|=j \le k$.
\end{proposition}

\begin{proof}
This fact comes trivially from the $k$ sparse initial samples $U$. As explained before, $\mathcal{G}$ is composed of at most $HW$ sub-graphs $\mathcal{G}^i$. Because we initialize the algorithm with $k \ll HW$ seeds, these can at most span $k$ sub-graphs each leading to a single reciprocal match. Due to the potential presence of \orange{convergence basins}, as seen in~\cref{fig:graph}, samples can merge along the paths to their root cycles, decreasing the final number of reciprocals and explaining the inequality $j \le k$. 
\end{proof}

\begin{figure*}[h!]
    \centering
    \includegraphics[width=0.7\linewidth]{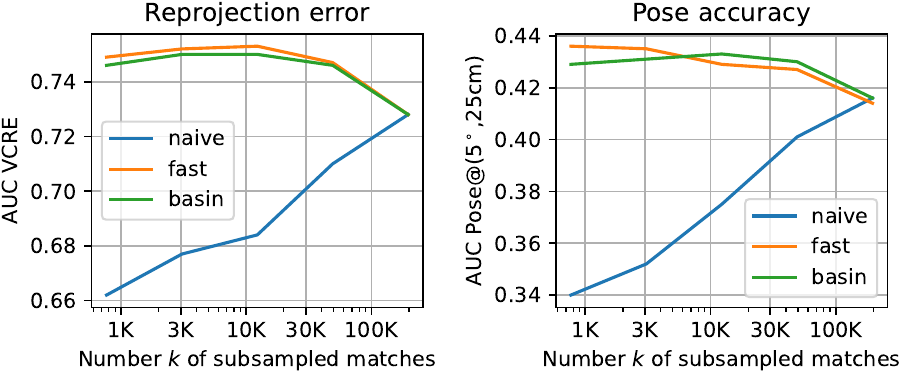}
    \caption{Comparison of the performance on the Map-free benchmark (validation set) for different subsampling approaches: `naive' denotes the random uniform subsampling of the original full set of reciprocal matches; `fast' denotes the proposed fast reciprocal matching; and `basin' denotes random subsampling weighted by the size of the convergence basin. The `fast' and `basin' strategies perform similarly whereas naive subsampling leads to catastrophic results. }
    \label{fig:cmp_subsampling}
\end{figure*}

\subsection{Performance improves with fast~matching}

As observed in Figure~2 of the main paper, FRM significantly improves the performance.
In the minimal example we provide in~\cref{fig:graph}, it is clearly visible that the FRM provides a sampling biased towards finding reciprocal matches with large basins (bottom), since a greater number of initial samples can fall onto them compared to small basins (top). 
Note that the size of the basin is inversely proportional to the maximal density of reciprocal matches.
Interestingly with the FRM, this results in a more homogeneous distribution (\ie spatial coverage) of reciprocal matches than the full matching, as depicted in~\cref{fig:matching_density}.  
As a direct consequence of a more homogeneous spatial coverage, RANSAC is able to better estimate epipolar lines than when lots of points are packed together in a small image region, which in turn provides better and more stable pose estimates.

In order to demonstrate the effect of basin-biased sampling, we propose to compute the full correspondence set $\mathcal{M}$ (\cref{eq:corres2}) and to subsample it in two ways: first, we naively subsample it randomly to reach the same number of reciprocals as the FRM. Second, we compute the size of each basin (as shown in~\cref{fig:basins}) and we bias the subsampling using the sizes. We report the results of this experiment in~\cref{fig:cmp_subsampling}. While random subsampling results in catastrophic performance drops, basin-biased sampling actually increases the performance compared to using the full graph (rightmost datapoint). As expected, the FRM algorithm provides a performance that closely follows biased subsampling, yet by only a fraction of the compute compared to basin-biased sampling which requires to compute all reciprocal matches in order to measure basin sizes. Importantly, these observations hold for both reprojection error and pose accuracy, regardless of the variant of RANSAC used to estimate relative poses.

\begin{table*}[h!]
    \caption{Coarse matching compared to Coarse-to-Fine for the tasks of visual localization on Aachen Day-Night (left) and MVS reconstruction on the DTU dataset (right).    %
    }
    \label{tab:c2f_combo_table}
    \begin{center}
    \small
    \resizebox{0.53\textwidth}{!}{
    \begin{tabular}{lccc}
    \toprule
    Methods & Coarse-to-Fine & Day & Night       \\
    \midrule
    \master~top1 & $\times$ & 74.9/90.3/98.5 & 55.5/82.2/95.8  \\
    \master~top1 & $\checkmark$ & 79.6/93.5/98.7 & 70.2/88.0/97.4 \\
    \midrule
    \master~top20 & $\times$ & 80.8/93.8/99.5 & 74.3/92.1/100 \\
    \master~top20 & $\checkmark$ & 83.4/95.3/99.4 & 76.4/91.6/100 \\
    \bottomrule
    \end{tabular}
    }
    \resizebox{0.45\textwidth}{!}{
    \begin{tabular}{llccc}
    \toprule
    & Methods &  Acc.$\downarrow$ & Comp.$\downarrow$ & Overall$\downarrow$       \\
    \midrule
    & \duster~\cite{dust3r}  &   2.677  &  0.805  & 1.741  \\ %
    & \master{} Coarse &   0.652    &  0.592 &  0.622  \\ %
    & \master &   0.403   &  0.344  &  0.374 \\ %
    \bottomrule
    \end{tabular}
    }
    \normalsize
    \end{center}
\end{table*}

\section{Coarse-to-Fine}
\label{sec:coarse-to-fine-abl}

In this section, we showcase the important benefits of the \emph{coarse-to-fine} strategy. We compare it to \emph{coarse-only} matching, that simply computes correspondences on input images down-scaled to the resolution of the network.  

\mypar{Visual localization on Aachen Day-Night\cite{aachen}.}
For this task, the input images are of resolution $1600 \times 1200$ and $1024 \times 768$, in both landscape and portrait are downscaled to $512 \times 384$/$384 \times 512$.
We report the percentage of successfully localized images within three thresholds: (0.25m, 2°), (0.5m, 5°) and (5m, 10°) in~\cref{tab:c2f_combo_table} (left).
We observe significant performance drops when using coarse matching only, by up to 15\% in top1 on the Night split.

\mypar{MVS.}
The input images of the DTU dataset~\cite{dtu} are of resolution $1200 \times 1600$ downscaled to $384 \times 512$.
As in the main paper, we report here the accuracy, completeness and Chamfer distance of triangulated matches obtained with \master{}, in the \emph{coarse-only} and \emph{coarse-to-fine} settings in~\cref{tab:c2f_combo_table} (right).
While coarse matching still outperforms the direct regression of \duster{}, we see a clear drop in reconstruction quality in all metrics, nearly doubling the reconstruction errors.

\section{Detailed experimental settings}

In our experiments, we set the confidence loss weight $\alpha=0.2$ as in \cite{dust3r}, the matching loss weight $\beta=1$, local feature dimension $d=24$ and the temperature in the InfoNCE loss to $\tau=0.07$.
We report the detailed hyper-parameter settings we use for training \master{} in Table~\ref{tab:training_step234}.

\begin{table}[h]
    \caption{\textbf{Detailed hyper-parameters} for the training}
    \label{tab:training_step234}
    \centering
    \resizebox{0.9\linewidth}{!}{
    \begin{tabular}{l@{\hskip 1.0cm}l}
    \specialrule{1.5pt}{0.5pt}{0.5pt} 
    Hyper-parameters & fine-tuning  \\
    \midrule
    Optimizer & AdamW \\
    Base learning rate & 1e-4 \\
    Weight decay & 0.05 \\
    Adam $\beta$ & (0.9, 0.95 )\\
    Pairs per Epoch & 650k \\
    Batch size & 64 \\
    Epochs & 35 \\
    Warmup epochs & 7 \\
    Learning rate scheduler & Cosine decay \\
    \specialrule{1pt}{0pt}{0pt} 
    Input resolutions & $512{\times}384$, $512{\times}336$ \\
                                       & $512{\times}288$, $512{\times}256$ \\
                                       & $512{\times}160$ \\
    \midrule
    Image Augmentations & Random crop, color jitter \\
    \midrule
    Initialization & \duster{}~\cite{dust3r} \\
    \specialrule{1.5pt}{0.5pt}{0.5pt} 
    \end{tabular}
    }
\end{table}